\documentclass[11pt]{article}
\usepackage[margin=1in]{geometry}
\usepackage{amssymb,amsfonts,amsmath,amsthm,amscd,dsfont,mathrsfs,bbold,pifont}
\usepackage{blkarray}
\usepackage{graphicx,float,psfrag,epsfig,color}

\usepackage{subcaption}
\usepackage{comment}

\usepackage{algorithm}
\usepackage[noend]{algpseudocode}

\makeatletter
\def\BState{\State\hskip-\ALG@thistlm}
\makeatother

	\usepackage{microtype}
	\usepackage[pdftex,pagebackref=true,colorlinks]{hyperref}
	\hypersetup{linkcolor=[rgb]{.7,0,0}}
	\hypersetup{citecolor=[rgb]{0,.7,0}}
	\hypersetup{urlcolor=[rgb]{.7,0,.7}}

\newcommand{\Enote}[1]{\begin{center}\fbox{\begin{minipage}{35em}
                        {{\bf Emmanuel Note:} {#1}} \end{minipage}}\end{center}}

\newcommand\independent{\protect\mathpalette{\protect\independent}{\perp}} 
\def\independent#1#2{\mathrel{\rlap{$#1#2$}\mkern2mu{#1#2}}}

\newcommand{\F}{\mathcal{F}} 
\newcommand{\mR}{\mathbb{R}}

\newcommand{\pp}{\mathbb{P}}
\newcommand{\E}{\mathbb{E}}
\DeclareMathOperator{\Var}{Var}
\newcommand{\e}{\varepsilon}

\newcommand{\X}{\mathcal{X}}
\newcommand{\Y}{\mathcal{Y}}
\newcommand{\Z}{\mathcal{Z}}

\newcommand{\1}{\mathbb{1}}
\newcommand{\B}{\mathbb{B}}

\newcommand{\tX}{\tilde{X}}
\newcommand{\tY}{\tilde{Y}}
\newcommand{\tW}{\tilde{W}}

\newtheorem{theorem}{Theorem}

\newtheorem{lemma}{Lemma}
\newtheorem{corollary}{Corollary}
\newtheorem{definition}{Definition}
\newtheorem{remark}{Remark}
\newtheorem{example}{Example}

\begin{document}



\title{Provable limitations of deep learning}

\author{Emmanuel Abbe \\ EPFL
\and  Colin Sandon \\ MIT 
}

\date{}
\maketitle

\begin{abstract}
As the success of deep learning reaches more grounds, one would like to also envision the potential limits of deep learning. This paper gives a first set of results proving that certain deep learning algorithms fail at learning certain efficiently learnable functions. 
The results put forward a notion of cross-predictability (on function distributions) that characterizes when such failures take place. Parity functions provide an extreme example with a cross-predictability that decays exponentially, while a mere super-polynomial decay of the cross-predictability is shown to be sufficient to obtain failures. Examples in community detection and arithmetic learning are also discussed.

Recall that it is known that the class of neural networks (NNs) with polynomial network size can express any function that can be implemented in polynomial time, and that their sample complexity scales polynomially with the network size. Thus NNs have favorable approximation and estimation errors. The challenge is with the optimization error, as the ERM is NP-hard and there is no known efficient training algorithm with provable guarantees. The success behind deep learning is to train deep NNs with descent algorithms, such as coordinate, gradient or stochastic gradient descent. 

The failures shown in this paper apply to training poly-size NNs on function distributions of low cross-predictability with a descent algorithm that is either run with limited memory per sample or that is initialized and run with enough randomness (such as exponentially small Gaussian noise for GD). We further claim that such types of constraints are necessary to obtain failures, in that exact SGD with careful non-random initialization can be shown to learn parities. The cross-predictability in our results plays a similar role the statistical dimension in statistical query (SQ) algorithms, with distinctions explained in the paper. 
The proof techniques are based on exhibiting algorithmic constraints that imply a statistical indistinguishability between the algorithm's output on the test  model v.s.\ a null model, using information measures to bound the total variation distance. This is then translated into an algorithmic failure based on the limitations of the null model.

\end{abstract}

\newpage
\tableofcontents

\newpage

\section{Introduction}

It is known that the class of neural networks (NNs) with polynomial network size can express any function that can be implemented in a given polynomial time, and that their sample complexity scales polynomially with the network size. Thus NNs have favorable approximation and estimation errors. The main challenge is with the optimization error, as there is no known efficient training algorithm for NNs with provable guarantees, in particular, it is NP-hard to implement the ERM rule  \cite{net_hard,daniely}.

The success behind deep learning is to train {\it deep} NNs with {\it descent} algorithms (e.g., coordinate, gradient or stochastic gradient descent); this gives record performances in image \cite{imagenet}, speech \cite{speech} and document recognitions \cite{document}, and the scope of applications is increasing on a daily basis \cite{deep_nature,deep_book}. While deep learning operates in an overparametrized regime, and while SGD optimizes a highly non-convex objective function, the training by SGD gives astonishingly low generalization errors in these applications. A major research effort is devoted to explaining these successes, with various components claimed responsible, such as the compositional structure of neural networks matching that of real signals, the implicit regularizations behind SGD (e.g., its stochastic component), the increased size of data sets and the augmented computational power, among others. 

With the wide expansion of the field, one would like to also envision the potential limits of deep learning.
This is the focus of this paper.  To understand the limitations of deep learning, we look for classes of functions that are efficiently learnable by some algorithm, but not for deep learning. 

The function that computes the parity of a Boolean vector is a well-known candidate \cite{lecun_perso}, as most functions that SGD is likely to try would be essentially uncorrelated with it, making it difficult to get close enough to the right function in a manageable time. 
However, any Boolean function that can be computed in time $O(T(n))$ can also be expressed by a neural network of size $O(T(n)^2)$ \cite{parberry,shaishai}, and so one could always start with a neural net that is set to compute the desired function, such as the parity function.
The problem is thus meaningful only if one constraints the type of initialization (e.g., random initializations) or if one deals with a class of functions (concept class) rather than a specific one, as commonly done for parities  \cite{ohad2,ran_memory}. We next discuss the example of parities before going back to general function distributions in Sections \ref{insight}  and \ref{results}.


\subsection{Learning parities}\label{model}
The problem of learning parities is  formulated as follows.
Define the class of all parity functions by $\F=\{ p_s:  s \subseteq [n] \}$, where $p_s: \{+1,-1\}^n  \to \{+1,-1\}$ is such that $$p_s(x)=\prod_{i  \in s} x_i.$$ 
Nature picks $S$ uniformly at random in $2^{[n]}$, and with access to $\F$ but not to $S$, the problem is to run a descent algorithm for a polynomial number of steps $t$ (in $n$) to obtain $w^{(t)}$ (e.g., coordinate, gradient or stochastic gradient descent using labeled samples $(X_i,P_S(X_i))$ where the $X_i$'s are independently and uniformly drawn in $\{+1,-1\}^n$).


The goal is to have the neural network output a label $eval_{w^{(t)}}(X)$ on a uniformly random input $X$ that (at least) correlates with the true label $p_S(X)$, such as    
$$I(eval_{w^{(t)}}(X); p_S(X)) =\Omega_n(1),$$
for some notion of mutual information (e.g., TV, KL or Chi-squared mutual information),
or $$\pp\{eval_{w^{(t)}}(X) = p_S(X)\} =1/2+\Omega_n(1),$$ 
if the output is made binary. 


Note that this is a weak learning requirement, thus failing at this is discarding any stronger requirements related to PAC-learning as mentioned in Section \ref{results}.   
Note also that this objective can be achieved if we do not restrict ourselves to using a NN trained with a descent algorithm. In fact, one can simply take an algorithm that builds a basis from enough samples (e.g., $n + \Omega(\log(n))$) and solves the resulting system of linear equations to reconstruct $S$. This seems however far from how deep learning proceeds. For instance, SGD is ``memoryless'' in that it updates the weights of the NN at each step with a sample but does not a priori explicitly remember the previous samples. Since each sample gives very little information about the true $S$, it thus seems unlikely for SGD to make any progress on a polynomial time horizon. However, it is far from trivial to argue this formally if we allow the NN to be arbitrarily large and with arbitrary initialization (albeit of polynomial complexity), and in particular inspecting the gradient is typically not sufficient. 
In fact, we claim that this is wrong, and deep learning can learn the parity function with a careful (though poly-time) initialization --- See Sections \ref{positive} and \ref{universality}.

On the other hand, if the initialization is done at random, as commonly assumed \cite{shaishai}, and the descent algorithm is run with perturbations, as sometimes advocated in different forms \cite{perturbed_sgd,langevin1,langevin2}, or if one does not move with the full gradient such as in (block-)coordinate descent or more generally bounded-memory update rules, then we show that it is in fact hard to learn parities. We will provide results in such settings showing the failure of deep learning.

Note also that having GD run with little noise is not equivalent to having noisy labels for which learning parities can be hard irrespective of the algorithm used \cite{parity_blum,regev}; in fact, the amount of noise that we need for running GD to obtain failure is exponentially small, which would effectively represent no noise if that noise was added on the labels itself (e.g., Gaussian elimination would still work).


\subsection{An illustrative experiment}\label{exp}
To illustrate the phenomenon, we consider the following data set and numerical experiment in PyTorch \cite{paszke2017automatic}. The elements in $\X$ are images with a white background and either an even or odd number of black dots, with the parity of the dots determining the label  --- see Figure \ref{images}. 
The dots are drawn by building a $k \times k$ grid with white background and activating each square with probability $1/2$.

We then train a neural network to learn the parity label of these images. The architecture is a 3 hidden linear layer perceptron with 128 units and ReLU non linearities trained using binary cross entropy. The training and testing dataset are composed of 1000 images of grid-size $k=13$. We used PyTorch implementation of SGD with step size 0.1 and i.i.d.\ rescaled uniform weight initialization \cite{imagenet2}.


\begin{figure}[H]
\centering
  \includegraphics[width=0.5\linewidth]{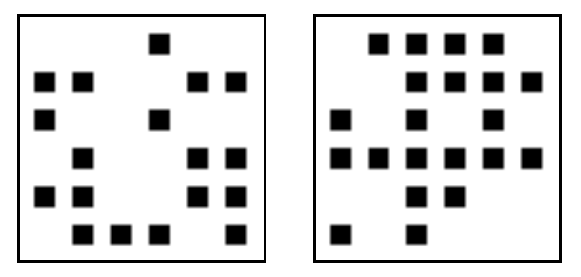}
\caption{Two images of $13^2=169$ squares colored black with probability $1/2$. The left (right) image has an even (odd) number of black squares. The experiment illustrates the incapability of deep learning to learn the parity.}
\label{images}
\end{figure}

Figure \ref{accuracy} show the evolution of the training loss, testing and training errors. As can be seen, the net can learn the training set but does not  generalize better than random guessing.  

\begin{figure}[H]
\centering
  \includegraphics[width=.9\linewidth]{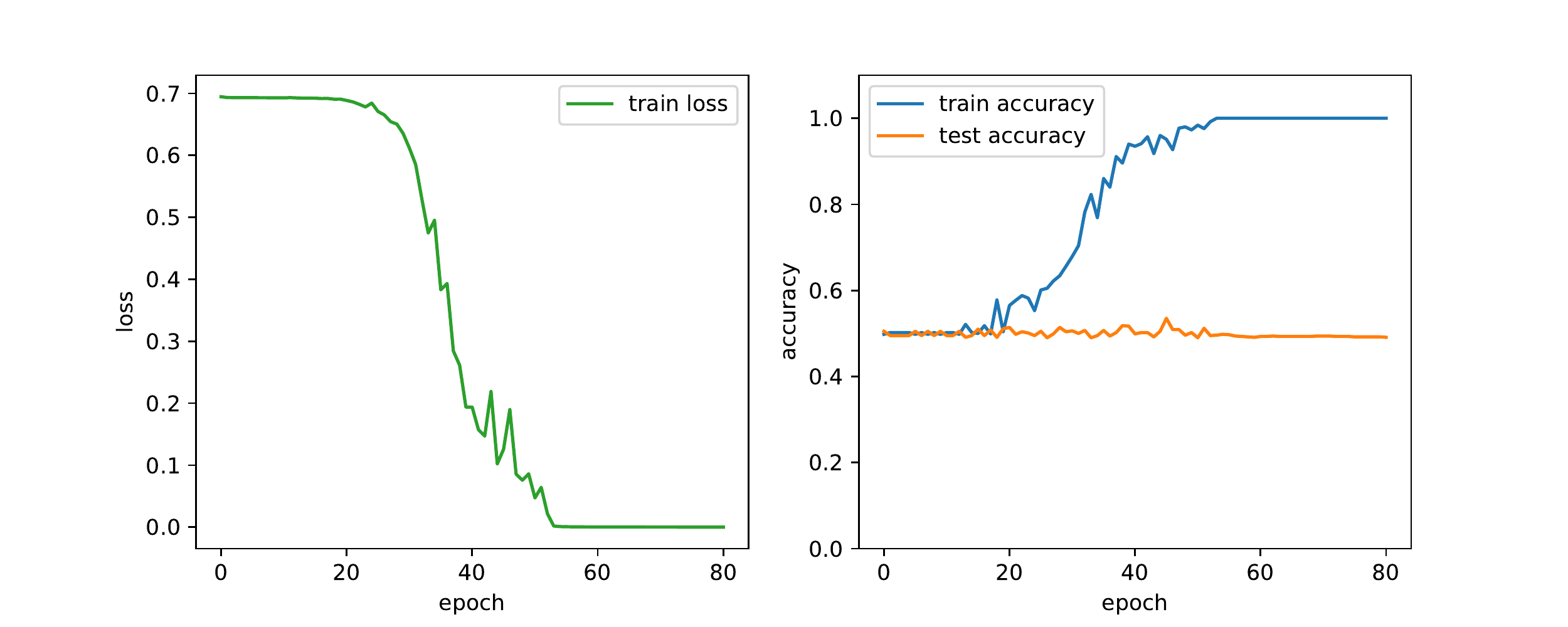}
\caption{Training loss (left) and training/testing errors (right) for up to 80 SGD epochs.}
\label{accuracy}
\end{figure}

Note that this is not exactly the model of Section \ref{model}. 
In the experiment, each image can be viewed as a Boolean vector of dimension  $k^2$, but the parity is taken over the entire vector, rather than a subset $S$.
This is however similar to our setup due to the following observation.\footnote{Another minor distinction is to sample from a pre-set training set v.s.\ sampling  fresh samples, but both can be implemented similarly for the experiment.} 
First consider the same experiment where one only takes the parity on the set $S$ consisting of the first half of the image; this would have the same performance outcome. Now, since the net is initialized with i.i.d.\ random weights, taking $S$ has the first half or any other random subset of $\approx k^2/2$ components leads to the same outcome by symmetry. Therefore, we expect failure for the same reason as we expect failure in our model with enough randomness in the initialization.

\subsection{Learning general function and input distributions}
In this paper, we will investigate the effect of general input distribution $P_\X$ and function distribution $P_\F$ on deep learning.

\begin{definition}
Let $n>0$, $\epsilon>0$, $P_{\X}$ be a probability distribution on $\X=\mathcal{D}^n$ for some set $\mathcal{D}$, and $P_\F$ be a probability distribution on the set of functions from $\X$ to $\{+1,-1\}$. 

Consider an algorithm $A$ that, given access to some information about $P_\X$ and $F \sim P_\F$ (e.g., samples under $P_\X$ labelled by $F$) outputs a function $\hat{F}$.
Then $A$ learns $(P_\F,P_{\X})$ with accuracy $\alpha$ if  $\pp\{\hat{F}(X)= F(X)\} \ge \alpha$, where the previous probability is taken over $(X,F) \sim P_\X \times P_\F$ and any randomness potentially used by $\hat{F}$. In particular, we say that $A$ (weakly) learns $(P_\F,P_{\X})$ if it learns $(P_\F,P_{\X})$ with accuracy $1/2+\Omega_n(1)$.
\end{definition}

\subsection{Insights about failures and successes of deep learning}\label{insight}
Our negative results reveal a measure that captures when the considered deep learning algorithms fail.

For a probability measure $P_\X$ on the data domain $\X$, and a probability measure $P_{\F}$ on the class of functions $\F$ from $\X$ to $\Y=\{+1,-1\}$, we define the cross-predictability of $P_{\F}$ with respect to $P_\X$ by 
\begin{align}
\mathrm{Pred}( P_\X, P_{\F} ) = \E_{F,F'} (\E_{X} F(X) F'(X))^2, \quad (X,F,F') \sim P_\X \times  P_{\F} \times P_{\F}.
\end{align}



This measures how predictable a sampled function is from another one on a typical input, and our results primarily exploit a low cross-predictability to obtain negative learning results. In particular, one can obtain failure results for SGD with bounded memory per sample or noise, as well as GD with noise for any distribution of cross-predictability that vanishes at a super-polynomial rate. One can further express this property in terms of the Fourier-Walsh expansion of the functions, as $\mathrm{Pred}( P_\X, P_{\F} )=\E_{F,F'} \langle \hat{F}, \hat{F'} \rangle^2= \| \E_F \hat{F}^{\otimes 2} \|_2^2$. In particular, taking parities on any random subset of $k=\omega(1)$ components already suffices to make the cross predictability decay super-polynomially, and thus to imply the failure of learning. 

The main insight is as follows. All of the algorithms that we consider essentially take a neural net, attempt to compute how well the functions computed by the net and slightly perturbed versions of the net correlate with the target function, and adjust the net in the direction of higher correlation. If none of these functions have significant correlation with the target function, this will generally make little or no progress. The descent algorithm evolves in a flat minima where no significant progress is made in a polynomial time horizon.

Of course, for any specific target function, we could initialize the net to correlate with it. However, if the target function is randomly drawn from a class with negligible cross-predictability, and if one cannot operate with GD or SGD perfectly, then no function is significantly correlated with the target function with nonnegligible probability and a descent algorithm will generally fail to learn the function in polynomial time.


\subsubsection{Learning a fixed function with random initialization}
Consider a symmetric function distribution, i.e., one that is invariant under a permutation of the input variables (e.g., random monomials). 
If one cannot learn this distribution with any initialization of the net, one cannot learn any given function in the distribution support with a random i.i.d.\ initialization of the net. This is because a random i.i.d.\ initialization of the net is itself symmetric under permutation.        

Further, we claim that the cross-predictability measure can be used to understand when a given function $h$ cannot be learned in poly-time with GD/SGD on poly-size nets that are randomly initialized, {\it without} imposing noise or memory constraints on how GD/SGD are run.  

Namely, define the cross-predictability between a target function and a random neural net as 
\begin{align}
\mathrm{Pred}( P_\X, h , \mu_{NN} ) = \E_{G} (\E_{X} h(X) \mathrm{eval}_{G,f}(X))^2, 
\end{align}
where $(G,f)$ is a random neural net under the distribution $\mu_{NN}$, i.e., $f$ is a fixed non-linearity, $G$ is a random graph that consists
of complete bipartite\footnote{One could consider other types of graphs but a certain amount of randomness has to be present in the model.} graphs between consecutive layers of a poly-size NN, with weights i.i.d.\ centered Gaussian of variance equal to one over the width of the previous layer, and $X \sim P_\X$ is independent of $G$. We then claim that if such a cross-predictability decays super-polynomially, training such a  random neural net with a polynomial number of steps of GD or SGD will fail at learning {\it even without memory or noise constraints}.
Again, as mentioned above, if the target function is permutation invariant, it cannot be learned with a random initialization and noisy GD with  small random noise. So the claim is that the random initialization gives already enough randomness in one step to cover all the added randomness from noisy  GD.


\subsubsection{Succeeding with large  cross-predictability}
In the case of random degree $k$ monomials, i.e., parity functions on a uniform subset $S$ of size $k$ with uniform inputs, we will show that deep learning fails at learning under memory or noise constraints as soon as $k=\omega(1)$. This is because the cross-predictability scales as ${n \choose k}^{-1}$, which is already super-polynomial when $k=\omega(1)$.  

On the flip side, if $k$ is constant, it is not hard to show that GD can learn this function distribution by inputting all the ${n \choose k}$ monomials in the first layer (and for example the cosine non-linearity to compute the parity in one hidden layer). Therefore, for random degree $k$ monomials, the deep learning algorithms described in our theorems will succeed at learning {\it if and only if} $k=O(1)$.  Thus one can only learn ``local'' functions in that sense.


We believe that small cross-predictability does not take place for typical labelling functions concerned with images or sounds, where many of the functions we would want to learn are correlated both with each other and with functions a random neural net is reasonably likely to compute. For instance, the objects in an image will correlate with whether the image is outside, which will in turn correlate with whether the top left pixel is sky blue. A randomly initialized neural net is likely to compute a function that is nontrivially correlated with the last of these, and some perturbations of it will correlate with it more, which means the network is in position to start learning the functions in question.

Intuitively, this is due to the fact that images and image classes have more compositional structures (i.e., their labels are well explained by combining `local' features). Instead, parity functions of large support size, i.e., not constant size but growing size, are not well explained by the composition of local features of the vectors, and require more global operations on the input.   



\subsection{Succeeding beyond cross-predictability}




As previously mentioned, certain methods of training a neural net cannot learn a random function drawn from any distribution with a cross-predictability that goes to zero at a superpolynomial rate. This raises the question of whether or not these methods can successfully learn a random function drawn from any distribution with a cross-predictability that is at least the inverse of a polynomial.

The first obstacle to learning such a function is that some functions cannot be computed to a reasonable approximation by any neural net of polynomial size. A probability distribution that always yields the same function has a cross-predictability of $1$, but if that function cannot be computed with nontrivial accuracy by any polynomial-sized neural net, then any method of training such a net will fail to learn it. 

Now, assume that every function drawn from $P_{\F}$ can be accurately computed by a neural net with polynomial size. If $P_{\F}$ has an inverse-polynomial cross-predictability, then two random functions drawn from the distribution will have an inverse-polynomial correlation on average. In particular, there exist a family of functions $f_0$ and a constant $c$ such that if $F\sim P_{\F}$ then $\E_{F} (\E_{X} F(X) f_0(X))^2=\Omega(n^{-c})$. Now, consider a neural net $(G,\phi)$ that computes $f_0$. Next, let $(G',\phi)$ be the neural net formed by starting with $(G,\phi)$ then adding a new output vertex $v$ and an intermediate vertex $v'$. Also, add an edge of very low weight from the original output vertex to $v$ and an edge of very high weight from $v$ to $v'$. This ensures that changing the weight of the edge to $v'$ will have a very large effect on the behavior of the net, and thus that SGD will tend to primarily alter its weight. That would result in a net that computes some multiple of $f_0$. If we set the loss function equal to the square of the difference between the actual output and the desired output, then the multiple of $f_0$ that has the lowest expected loss when trying to compute $F$ is $\E_{X}[f_0(X)F(X)] f_0$, with an expected loss of $1-\E^2_{X}(f_0(X)F(X))$. We would expect that training $(G',\phi)$ on $F$ would do at least this well, and thus have an expected loss over all $F$ and $X$ of at most $1-\E_{F} (\E_{X} F(X) f_0(X))^2=1-\Omega(n^{-c})$. That means that it will compute the desired function with an average accuracy of $1/2+\Omega(n^{-c})$. Therefore, if the cross-predictability is polynomial, one can indeed learn with at least a polynomial accuracy.  


However, we cannot do much better than this. To demonstrate that, consider a probability distribution over functions that returns the function that always outputs $1$ with probability $1/\ln(n)$, the function that always outputs $-1$ with probability $1/\ln(n)$, and a random function otherwise. This distribution has a cross-predictability of $\theta(1/\ln^2(n))$. However, a function drawn from this distribution is only efficiently learnable if it is one of the constant functions. As such, any method of attempting to learn a function drawn from this distribution that uses a subexponential number of samples will fail with probability $1-O(1/\ln(n))$. In particular, this type of example demonstrates that for any $g=o(1)$, there exists a probability distribution of functions with a cross-predictability of at least $g(n)$ such that no efficient algorithm can learn this distribution with an accuracy of $1/2+\Omega(1)$.

However, one can likely prove that a neural net trained by noisy GD or noisy SGD can learn $P_{\F}$ if it satisfies the following property. Let $m$ be polynomial in $n$, and assume that there exists a set of functions $g_1,...,g_m$ such that each of these functions is computable by a polynomial-sized neural net and the projection of a random function drawn from $P_{\F}$ onto the vector space spanned by $g_1,...,g_m$ has an average magnitude of $\Omega(1)$. In order to learn $P_{\F}$, we start with a neural net that has a component that computes $g_i$ for each $i$, and edges linking the outputs of all of these components to its output. Then, the training process can determine how to combine the information provided by these components to compute the function with an advantage that is within a constant factor of the magnitude of its projection onto the subspace they define. That yields an average accuracy of $1/2+\Omega(1)$. However, we do not think that this is a necessary condition to be able to learn a distribution using a neural net trained by noisy SGD or the like.



\subsection{Related works}\label{related}

The difficulty of learning parities with NNs is not new. The parity was already known to be hard based on the early works on the perceptron \cite{perceptron}, see also \cite{hastad,allender}. We now discuss various works related to ours.

{\bf Statistical querry algorithms.} The lack of correlations between two parity functions and its implication in learning parities appear also in the context of statistical query learning algorithms \cite{kearns,query}, which are algorithms that access estimates of the expected value of some query function over the underlying data distribution (e.g., the first moment statistics of inputs' coordinates). 

In particular, gradient-based algorithms with approximate oracle access are realizable as statistical query algorithms, and \cite{kearns} implies that the class of parity functions cannot be learned by such algorithms, which  
gives a result similar in nature to our Theorem \ref{thm2'}. 
However, the result from \cite{kearns} and its generalization in \cite{parity_blum} have a few differences from those presented  here; first these papers define successful learning for {\it any} function in a class of function, where as we will work with {\it typical} functions from a function distributions; second these papers require the noise to be adversarial, while we use statistical (and thus less restrictive) noise; finally the proof techniques are different, mainly based on Fourier analysis in \cite{parity_blum} and on hypothesis testing here. 

One could also use \cite{kearns} to obtain a variant of our Theorem \ref{thm2}. 
Technically \cite{kearns} only says that a SQ algorithm with a polynomial number of queries and inverse polynomial noise cannot learn a parity function, but the proof would still work with appropriately chosen exponential parameters. To further convert this to the setting with statistical noise, one could use an argument saying that the Gaussian noise is large enough to mostly drown out the adversarial noise if the latter is small enough, but the resulting bounds would be slightly looser than ours because that would force one to make trade offs between making the amount of adversarial noise in the SQ result low and minimizing the probability that one of the queries does provide meaningful information. Alternately, one could probably rewrite their proof using Gaussian noise instead of bounded adversarial noise and bound sums of $L_1$ differences between the probability distributions corresponding to different functions instead of arguing that with high probability the bound on the noise quantity is high enough to allow the adversary to give a generic response to the query. 

To see how Theorem \ref{thm2'} departs from the setting of \cite{parity_blum} beyond the statistical noise discussed above, note that the cross-predictability captures the expected inner product $\langle F_1,F_2 \rangle_{P_\X}$ over two i.i.d.\ functions $F_1,F_2$ under $P_\F$, whereas the statistical dimension defined in \cite{parity_blum} is the largest number $d$ of functions $f_i \in \F$ that are nearly orthogonal, i.e, $|\langle f_i,f_j \rangle_{P_\X}| \le 1/d^3$, $1 \le i< j \le d$. Therefore, while the cross-predictability and statistical dimension tend to be negatively correlated, one can construct a family $\F$ that contains many almost orthogonal functions, yet with little mass under $P_\F$ on these so that the distribution has a high cross-predictability.\footnote{
For example, take a class containing two types of functions, hard and easy, such as parities on sets of components and almost-dictatorships which agree with the first input bit on all but $n$ of the inputs. The parity functions are orthogonal, so the union contains a set of size  $2^n$ that is pairwise orthogonal. However, there are about $2^n$ of the former and $2^{n^2}$ of the latter, so if one picks a function uniformly at random on the union, it will belong to the latter group with high probability, and the cross-predictability will be $1-o(1)$.
} So one can build examples of function classes where it is possible to learn with a moderate cross-predictability while the statistical dimension is large and learning fails in the sense of \cite{parity_blum}. 


There have been further extensions of the statistical-dimension, such as in \cite{vempala2} which allows for a probability measure on the functions as well. The statistical dimension as defined in Definition 2.6 of \cite{vempala2} measures the maximum probability subdistribution with a sufficiently high correlation among its members.\footnote{Another difference with \cite{vempala2} is that our results show failures for weak rather than exact learning.} As a result, any probability distribution with a low cross predictability must have a high statistical dimension in that sense. However, a distribution of functions that are all moderately correlated with each other could have an arbitrarily high statistical dimension despite having a reasonably high cross-predictability. For example, using definition 2.6 of \cite{vempala2} with constant $\bar{\gamma}$ on the collection of functions from $\{0,1\}^n->\{0,1\}$ that are either 1 on $1/2+\sqrt{\gamma}/4$ of the possible inputs or 1 on $1/2-\sqrt{\gamma}/4$ of the inputs, gives a statistical dimension with average correlation $\gamma$ that is doubly exponential in $n$. However,this has a cross predictability of $\gamma^2/16$.

One could imagine a way to prove Thm \ref{thm1'} with prior SQ works as follows, (i) generalize the paper of \cite{parity_conj} that establishes a result similar to our Thm \ref{thm1'} for the special case of parities to the class of low cross-predictability functions, (ii) show that this class has the right notion of statistical dimension that is high. However, the distinction between low cross-predictability and high stat-dimensional would kick in at this point. If we take the example mentioned in the previous paragraph, the version of SGD used in Theorem \ref{thm1'} could learn to compute a function drawn from this distribution with expected accuracy $1/2+\sqrt{\gamma}/8$ given $O(1/\gamma)$ samples, so the statistical dimension of the distribution is not limiting its learnability by such algorithms in an obvious way. One might be able to argue that a low cross-predictability implies a high statistical dimension with a value of $\gamma$ that vanishes sufficiently quickly and then work from there. However, it is not clear exactly how one would do that, or why it would give a preferred approach.

Paper \cite{query2} also shows that gradient-based algorithms with approximate oracle access are realizable as statistical query algorithms, however, \cite{query2} makes a convexity assumption that is not satisfied by non-trivial neural nets. SQ lower bounds for learning with data generated by neural networks is also investigated in \cite{song} and for neural network models with one hidden nonlinear activation layer in \cite{wilmes}.


Finally, the current SQ framework does not apply to noisy SGD (even for adversarial noise). In fact, we believe that it is possible to learn parities with better noise-tolerance and complexity than any SQ algorithm will do --- see further discussion below and in Sections \ref{positive} and \ref{universality}.


{\bf Memory-sample trade-offs.} In \cite{ran_memory}, it is shown that one needs either quadratic memory or an exponential number of samples in order to learn parities, settling a conjecture from \cite{parity_conj}. This gives a  non-trivial lower bound on the number of samples needed for a learning problem and a first complete negative result in this context, with applications to bounded-storage cryptography  \cite{ran_memory}. 
Other works have extended the results of \cite{ran_memory}; in particular \cite{ran_sparse} applies to k-sparse sources, \cite{ran17} to other functions than parities, and \cite{grt} exploits properties of two-source extractors to obtain comparable memory v.s.\ sample complexity trade-offs, with similar results obtained in \cite{bogy}. The cross-predictability has also similarity with notions of almost orthogonal matrices used in $L_2$-extractors for two independent sources \cite{two_source,grt}.

    In contrast to this line of works (i.e., \cite{ran_memory} and follow-up papers), our Theorem \ref{thm1'} specialized to the parity functions shows that one needs exponentially many samples to learn parities if less than $n/24$ pre-assigned bits of memory are used {\it per sample}. These are thus different models and results. Our result does not say anything interesting about our ability to learn parities with an algorithm that has free access to memory, while the result of \cite{ran_memory} says that it would need to have $\Omega(n^2)$ total memory or an exponential number of samples. On the flip side, our result shows that an algorithm with unlimited amounts of memory will still be unable to learn a random parity function from a subexponential number of samples if there are sufficiently tight limits on how much it can edit the memory while looking at each sample. The latter is relevant to study SGD with bounded memory as discussed in this paper. 
    
    Note also that for the special case of parities, one could aim for Theorem \ref{thm1'} using \cite{parity_conj} with the following argument. If bounded-memory SGD could learn a random parity function with nontrivial accuracy, then we could run it a large number of times, check to see which iterations learned it reasonably successfully, and combine the outputs in order to compute the parity function with an accuracy that exceeded that allowed by Corollary 4 in \cite{parity_conj}. However, in order to obtain a generalization of this argument to low cross-predictability functions, one would need to address the points made previously regarding Theorem \ref{thm1'} and \cite{parity_conj} (point  (i) and (ii)).

{\bf Gradient concentration.} Finally, \cite{ohad}, with an earlier version in \cite{ohad2} from the first author, also give strong support to the impossibility of learning parities. 
In particular the latter discusses whether specific assumptions on the ``niceness'' of the input distribution or the target function (for example based on notions of smoothness, non-degeneracy, incoherence or random choice of parameters), are sufficient to guarantee learnability using gradient-based methods, and evidences are provided that neither class of assumptions alone is sufficient.

\cite{ohad} gives further theoretical insights  and practical experiments on the failure of learning parities in such context. More specifically, it proves that the gradient of the loss function of a neural network will be essentially independent of the parity function used. This is achieved by a variant of our  Lemma \ref{new-pred} below with the requirement in \cite{ohad} that the loss function is 1-Lipschitz\footnote{The proofs are both simple but slightly different, in particular Lemma \ref{new-pred} does not make  regularity assumptions.}. This provides a strong intuition of why one should not be able to learn a random parity function using gradient descent or one of its variants, and this is backed up with theoretical and experimental evidence, bringing up the issue of the flat minima. However, it is not proved that one cannot learn parity using stochastic gradient descent or the like. The implication is far from trivial, as with the right algorithm, it is indeed possible to reconstruct the parity function from the gradients of the loss function on a list of random inputs. In fact, as mentioned above and further discussed in this paper, we believe that it is possible to learn a random parity function in polynomial time by using GD or SGD with a careful poly-time initialization of the net (that is of course also agnostic to the parity function). As we show further here, obtaining formal negative results requires more specific assumptions and elaborate proofs, already for GD and particularly for SGD.


\section{Results}\label{results}

\subsection{Definitions and models}

Before we can talk about the effectiveness of deep learning at learning parity functions, we have to establish some basic notions about deep learning. First of all, in this paper we will be using the following definition for a neural net.

\begin{definition}
A neural net is a pair of a function $f:\mathbb{R}\rightarrow \mathbb{R}$ and a weighted directed graph $G$ with some special vertices and the following properties. First of all, $G$ does not contain any cycle. Secondly, there exists $n>0$ such that $G$ has exactly $n+1$ vertices that have no edges ending at them, $v_0$, $v_1$,...,$v_n$. We will refer to $n$ as the input size, $v_0$ as the constant vertex and $v_1$, $v_2$,..., $v_n$ as the input vertices. Finally, there exists a vertex $v_{out}$ such that for any other vertex $v'$, there is a path from $v'$ to $v_{out}$ in $G$.  We also use denote by $w(G)$ the weights on the edges of $G$. 
\end{definition}

\begin{definition}
Given a neural net $(f,G)$ with input size $n$, and $x\in\mathbb{R}^n$, the evaluation of $(f,G)$ at $x$, written as $eval_{(f,G)}(x)$, is the scalar computed by means of the following procedure:
\begin{enumerate}
\item Define $y\in \mathbb{R}^{|G|}$ where $|G|$ is the number of vertices in $G$, set $y_{v_0}=1$, and set $y_{v_i}=x_i$ for each $i$.

\item Find an ordering $v'_1,...,v'_m$ of the vertices in $G$ other than the constant vertex and input vertices such that for all $j>i$, there is not an edge from $v'_j$ to $v'_i$.

\item For each $1\le i\le m$, set 
\[y_{v'_i}=f\left(\sum_{v: (v,v'_i)\in E(G)} w_{v,v'_i} y_v\right)\]

\item Return $y_{v_{out}}$.
\end{enumerate}
\end{definition}

Generally, we want to find a neural net that computes a certain function, or at least a good approximation of that function. A reasonable approach to doing that is to start with some neural network and then attempt to adjust its weights until it computes a reasonable approximation of the desired function. The trademark of deep learning is to do this by defining a loss function in terms of how much the network's outputs differ from the desired outputs, and then use a descent algorithm to try to adjust the weights. More formally, if our loss function is $L$, the function we are trying to learn is $h$, and our net is $(f,G)$, then the net's loss at a given input $x$ is $L(h(x)-eval_{(f,G)}(x))$ (or more generally $L(h(x),eval_{(f,G)}(x))$). Given a probability distribution for the function's inputs, we also define the net's expected loss as $\E[L(h(X)-eval_{(f,G)}(X))]$. We now discuss three common descent algorithms.

The gradient descent (GD) algorithm does the following. Its input includes the initial neural net, a learning rate, and the number of time steps the algorithm runs for. At each time step, the algorithm computes the derivative of the net's expected loss with respect to each of its edge weights, and then decreases each edge weight by the derivative with respect to that weight times the learning rate. After the final time step, it returns the current neural net.

One problem with GD is that it requires computing an expectation over every possible input, which is generally impractical. One possible fix to that is to use stochastic gradient descent (SGD) instead of gradient descent. The input to SGD includes the initial neural net, a learning rate, and the number of time steps the algorithm runs for. However, instead of computing an expectation over all possible inputs in each time step, SGD randomly selects a single input in each time step, computes the derivative of the net's loss at that input with respect to each edge weight, and decreases every edge weight by the corresponding derivative times the learning rate. Note that the stochasticity in SGD is sometimes also claimed to help with the generalization error of deep learning; with the insight that it helps with stability, implicit regularization or bad critical points \cite{faster,rethink,bad_min,bad_min2}.

A third option is to use coordinate (or block-coordinate) descent (CD) instead of gradient descent. This works as SGD, except that in each time step CD only updates a small number of edge weights, with all other weights remaining fixed. There are multiple options for deciding which edge weights to change, some of which would base the decision on the chosen input, but our result involving CD will be generic and not depend on the details of how the edges are chosen.

It is also possible to use a noisy version of any of the algorithm mentioned above. This would be the same as the noise-free version, except that in each time step, the algorithm independently draws a noise term for each edge from some probability distribution and adds it to that edge's weight. Adding noise can help  avoid getting stuck in local minima or regions where the derivatives are small \cite{perturbed_sgd}, however it can also drown out information provided by the gradient, and some learning algorithms are extremely sensitive to noise.

Finally, each of these algorithms needs to start with an initial set of weights before initiating the descent.
A priori, the initialization could be done according to any rule, albeit with manageable complexity since we will focus in this paper on the efficiency of algorithms. However, in practice, SGD implementations in deep learning typically start with random initializations, see \cite{shaishai}, or variants such as \cite{imagenet2} that involve different types of probability  distributions for the initial weights. Note that random initializations have also been shown to help with escaping certain bad extremal points \cite{random_init}.

We want to answer the question of whether or not training a neural net with these algorithms is a universal method of learning, in the sense that it can learn anything that is reasonably learnable. We next define exactly what this means. 

\begin{definition}
Let $n>0$, $\epsilon>0$, $P_{\X}$ be a probability distribution on $\{0,1\}^n$, and $P_\F$ be a probability distribution on the set of functions from $\{0,1\}^n$ to $\{0,1\}$. Also, let $X_0, X_1,...$ be independently drawn from $P_{\X}$ and $F\sim P_\F$. An algorithm learns $(P_\F,P_{\X})$ with accuracy $1/2+\epsilon$ if the following holds. There exists $T>0$ such that if the algorithm is given the value of $(X_i, F(X_i))$ for each $i<T$ and it is given the value of $X_T$, it returns $Y_T$ such that $P[F(X_T)=Y_T]\ge 1/2+\epsilon$.
\end{definition}
In particular, we talk about  ``learning parities" in the case where $P_\F$ picks a parity function uniformly at  random and $P_\X$ is uniform on $\{+1,-1\}^n$, as defined in Section \ref{model}.

\begin{remark}
As we have defined them, neural nets generally give outputs in $\mathbb{R}$ rather than $\{0,1\}$. As such, when talking about whether training a neural net by some method learns a given function, we will implicitly be assuming that the output of the net on the final input is thresholded at some predefined value or the like. None of our results depend on exactly how we deal with this part.
\end{remark}



\begin{definition}
For each $n>0$, let\footnote{Note that these are formally sequences of distributions.} $P_{\X}$ be a probability distribution on $\{0,1\}^n$, and $P_{\F}$ be a probability distribution on the set of functions from $\{0,1\}^n$ to $\{0,1\}$. We say that $(P_\F,P_{\X})$ is efficiently learnable if there exists $\epsilon>0$, $N>0$, and an algorithm with running time polynomial in $n$ such that for all $n\ge N$, the algorithm learns $(P_{\F},P_{\X})$ with accuracy $1/2+\epsilon$.
\end{definition}

\subsection{Negative results}

In order to disprove the universality of learning of these algorithms, we need an efficiently learnable function that they fail to learn. In this paper, we will use a random parity function with input that is uniformly distributed in $\{0,1\}^n$. One can easily learn such a function by taking a linear sized sample of its values on random inputs, finding a basis of $\{0,1\}^n$ in the inputs sampled, and then using the fact that these parity functions are linear. However, as we will show, some of the algorithms listed above are unable to learn such a function. The fundamental problem is that any two different parity functions are uncorrelated, so no function is significantly correlated with a nonnegligible fraction of them. As a result, the neural nets will generally fail to even come close enough to computing the desired function for the gradient to provide useful feedback on how to improve it. We will formalize this idea by defining a quantity called cross-predictability and showing that it is exponentially small for random parity functions. Similar negative results would hold for other families of functions with comparably low cross-predictability. The results in question are the following:






\begin{theorem}\label{thm1'}
Let $\epsilon>0$, and $P_{\F}$ be a probability distribution over functions with a cross-predictability of $\mathrm{c_p}=o(1)$. For each $n > 0$, let $(f,g)$ be a neural net of polynomial size in $n$ such that each edge weight is recorded using $O(\log(n))$ bits of memory. Run stochastic gradient descent on $(f,g)$ with at most $\mathrm{c_p}^{-1/24}$ time steps and with $o(|\log(\mathrm{c_p})|/\log(n))$ edge weights updated per time step.  For all sufficiently large $n$, this algorithm fails at learning functions drawn from $P_{\F}$ with accuracy $1/2 + \epsilon$.
\end{theorem}

\begin{corollary}
Coordinate descent with a polynomial number of steps and precision fails at learning parities with non-trivial accuracy.   
\end{corollary}

\begin{remark}
Specializing previous theorem to the case of parities, one obtains the following. 
Let $\epsilon>0$. For each $n > 0$, let $(f,g)$ be a neural net of polynomial size in $n$ such that each edge weight is recorded using $O(\log(n))$ bits of memory. Run stochastic gradient descent on $(f,g)$ with at most $2^{n/24}$ time steps and with $o(n/\log(n))$ edge weights updated per time step.  For all sufficiently large $n$, this algorithm fails at learning parities with accuracy $1/2 + \epsilon$.

As discussed in Section \ref{related},
one could obtain the special case of Theorem \ref{thm1'} for parities using \cite{parity_conj} with the following argument. If bounded-memory SGD could learn a random parity function with nontrivial accuracy, then we could run it a large number of times, check to see which iterations learned it reasonably successfully, and combine the outputs in order to compute the parity function with an accuracy that exceeded that allowed by Corollary 4 in \cite{parity_conj}. However, in order to obtain a generalization of this argument to low cross-predictability functions, one would need to address the points made in Section \ref{related} regarding statistical dimension and cross-predictability.
\end{remark}

\begin{theorem}\label{thm2}
For each $n > 0$, let $(f,g)$ be a neural net of polynomial size in $n$. Run gradient descent on $(f,g)$ with less than $2^{n/10}$ time steps, a learning rate of at most $2^{n/10}$, Gaussian noise with variance at least $2^{-n/10}$ and overflow range of at most $2^{n/10}$. For all sufficiently large $n$, this algorithm fails at learning parities with accuracy $1/2 + 2^{-n/10}$.
\end{theorem}
See Section \ref{related} for how the above compares to \cite{kearns}; in particular, an application of \cite{kearns} would not give the above exponents.

More generally, we have the following result that applies to low cross-predictability functions (and beyond some cases of large statistical dimension --- see Section \ref{related}).

\begin{theorem}\label{thm2'}
 Let $P_\X,P_{\F}$ be such that the output distribution is balanced,\footnote{Non-balanced cases can be handled by modifying definitions appropriately.} i.e., $\pp\{F(X)=0\}=\pp\{F(X)=1\}$ when $(X,F) \sim P_\X \times P_{\F}$, and let $\mathrm{c_p}:=\mathrm{Pred}( P_\X, P_{\F} )$.  
For each $n > 0$, let $(\phi,g)$ be a neural net of size $|E(g)|$. Run gradient descent on $(\phi,g)$ with at most $T$ time steps, a learning rate of at most $\gamma$, Gaussian noise with variance at least $\sigma^2$ and an overflow range for the derivatives of at most $B$.
If $\max(T,|E(g)|,1/\sigma,B, \gamma) = n^{O(1)}$ and $\mathrm{c_p}= n^{-\omega(1)}$, 
this algorithm fails at learning functions drawn from $P_{\F}$ with accuracy $1/2 + \Omega_n(1)$.
\end{theorem}

\begin{corollary}
The polynomial deep learning system of previous corollary can weakly learn a random degree-$k$ monomial if and only if $k=O_n(1)$.
\end{corollary}

\begin{remark}
An overflow range of $B$ means that any value (e.g., derivatives of the loss function for a certain input) potentially exceeding $B$ (or $-B$) is kept at $B$ (or $-B$).
\end{remark}

\begin{remark}
We could alternately have defined the algorithm such that if there is any input for which one of the derivatives is larger than the overflow range, we give up and return random predictions. In this case, the result above would still hold.

As a third option, we could define $\epsilon$ to be the probability that there is a time step in which the derivative of the loss function with respect to some edge weight is greater than the overflow range for some input. In this case, given that no such overflow occurs, the algorithm would fail to learn parities with accuracy $1/2+2^{-n/10}/(1-\epsilon)$.
\end{remark}


\begin{remark}
Note that having GD run with a little noise is not equivalent to having noisy labels for which learning parities can be hard irrespective of the algorithm used \cite{parity_blum,regev}. The amount of noise needed for GD in the above theorem is exponentially small, and if such amount of noise were added to the sample labels themselves, then the noise would essentially be ineffective (e.g., Gaussian elimination would still work with rounding, or if the noise were Boolean with such variance, no flip would take place with high probability). The  failure is thus due to the nature of the algorithm.
\end{remark}

In the case of full gradient descent, the gradients of the losses with respect to different inputs mostly cancel out, so an exponentially small amount of noise is enough to drown out whatever is left. With stochastic gradient descent, that does not happen, and we have the following instead.

\begin{definition}
Let $(f,g)$ be a NN, and recall that $w(g)$ denotes the set of weights on the edges of $g$. Define the $\tau$-neighborhood of $(f,g)$ as 
\begin{align}
N_{\tau}(f,g)=\{(f,g'): E(g')=E(g),  |w_{u,v}(g)-w_{u,v}(g')| \le \tau , \forall (u,v) \in E(g) \}.    
\end{align}
\end{definition}

\begin{theorem}\label{thm3}
For each $n>0$, let $(f,g)$ be a neural net with size $m$ polynomial in $n$, and let $B,\gamma,T>0$ such that $B$, $1/\gamma$, and $T$ are polynomial in $n$. There exist $\sigma=O(m^2\gamma^2 B^2/n^2)$ and $\sigma'=O(m^3 \gamma^3 B^3/n^2)$ such that the following holds. Perturb the weight of every edge in the net by a Gaussian distribution of variance $\sigma$ and then train it with a noisy stochastic gradient descent algorithm with learning rate $\gamma$, $T$ time steps, and Gaussian noise with variance $\sigma'$. Also, let $p$ be the probability that at some point in the algorithm, there is a neural net $(f,g')$ in $N_{\tau}(f,g)$, $\tau=O(m^2\gamma B/n)$, such that at least one of the first three derivatives of the loss function on the current sample with respect to some edge weight(s) of $(f,g')$ has absolute value greater than $B$. Then this algorithm fails to learn parities with an accuracy greater than $1/2+2p+O(Tm^4B^2\gamma^2/n)$.
\end{theorem}


\begin{remark}
Normally, we would expect that if training a neural net by means of SGD works, then the net will improve at a rate proportional to the learning rate, as long as the learning rate is small enough. As such, we would expect that the number of time steps needed to learn a function would be inversely proportional to the learning rate. This theorem shows that if we set $T=c/\gamma$ for any constant $c$ and slowly decrease $\gamma$, then the accuracy will approach $1/2+2p$ or less. If we also let $B$ slowly increase, we would expect that $p$ will go to $0$, so the accuracy will go to $1/2$. It is also worth noting that as $\gamma$ decreases, the typical size of the noise terms will scale as $\gamma^{3/2}$. So, for sufficiently small values of $\gamma$, the noise terms that are added to edge weights will generally be much smaller than the signal terms. 
\end{remark}

\begin{remark}
The bound on the derivatives of the loss function is essentially a requirement that the behavior of the net be stable under small changes to the weights. It is necessary because otherwise one could effectively multiply the learning rate by an arbitrarily large factor simply by ensuring that the derivative is very large. Alternately, excessively large derivatives could cause the probability distribution of the edge weights to change in ways that disrupt our attempts to approximate this probability distribution using Gaussian distributions. For any given initial value of the neural net, and any given $M>0$, there must exists some $B$ such that as long as none of the edge weights become larger than $M$ this will always hold. However, that $B$ could be very large, especially if the net has many layers. 
\end{remark}

\begin{definition}\cite{mohri}
A class of function $\F$ is said to be PAC-learnable if there exists an algorithm $A$ and a polynomial function $poly(·,·,·,·)$ such that for any $\e > 0$ and $\delta > 0$, for all distributions $D$ on $\X$ and for any target function $f \in \F$, the following holds for any sample size $m \ge poly(1/\e, 1/\delta, n, size(f))$:
\begin{align}
\pp_{S\sim D^m}\{ \pp_{X \sim D}\{h_S(X)\ne f(X)\} \le \e \} \ge 1 - \delta.
\end{align}
If $A$ further runs in $poly(1/\e, 1/\delta, n, size(f))$, then $\F$ is said to be efficiently PAC-learnable. When such an algorithm $A$ exists, it is called a PAC-learning algorithm for $\F$.
\end{definition}

Therefore, picking $D$ to be uniform on $\B^n$, $\e=1/10$ and $\delta=1/10$,  Theorems \ref{thm1'}, \ref{thm2}, \ref{thm3} imply that a neural network trained by one of the specified descent algorithms on a polynomial number of samples will not compute the parity function in question with accuracy $1-\epsilon$ with probability $1-\delta$. Thus, we have the following.

\begin{corollary}
 Deep learning algorithms as described in Theorems \ref{thm1'}, \ref{thm2}, \ref{thm3}, fail at PAC-learning the class of parity functions $\F=\{ p_s:  s \subseteq [n] \}$ in poly$(n)$-time.
\end{corollary}

\subsection{Necessary limitations of negative results}\label{positive}

We show that deep learning fails at learning parities in polynomial time \text{if} the descent algorithm is either run with limited memory or initialized and run with enough randomness. However, if initialized carefully and run with enough memory and precision, we claim that it is in fact possible to learn parities: \\

\noindent
{\it  One can construct in polynomial time in $n$ a neural net $(f,g)$ that has polynomial size in $n$ such that for a learning rate $\gamma$ that is at most polynomial in $n$ and an integer $T$ that is at most polynomial in $n$, $(f,g)$ trained by SGD with learning rate $\gamma$ and $T$ time steps learns parities with accuracy $1-o(1)$.}\\
   
In fact, we claim that a more general result holds for any efficiently learnable distribution, and that deep learning can universally learn any such distribution (though the initialization will be unpractical) --- Section \ref{universality}. Full proofs will follow in subsequent versions of the paper.

\section{Other functions that are difficult for deep learning}\label{others}

This paper focuses on the difficulty of learning parities, with proofs extending to other function/input distributions having comparably low cross-predictability.  
Parities are not the most common type of functions used to generate real signals, but they are central to the construction of good codes (in particular the most important class of codes, i.e., linear codes, that rely heavily on parities). We mention now a few additional examples of functions that we believe would be also difficult to learn with deep learning. \\

\noindent
{\bf Arithmetic.} First of all, consider trying to teach a neural net arithmetic. More precisely, consider trying to teach it the following function. The function takes as input a list of $n$ numbers that are written in base $n$ and are $n$ digits long, combined with a number that is $n+1$ digits long and has all but one digit replaced by question marks, where the remaining digit is not the first. Then, it returns whether or not the sum of the first $n$ numbers matches the remaining digit of the final number. So, it would essentially take expressions like the following, and check whether there is a way to replace the question marks with digits such that the expression is true.
\begin{align*}
&120\\
+&112\\
+&121\\
=?&?0?
\end{align*}

Here, we can define a class of functions by defining a separate function for every possible ordering of the digits. If we select inputs randomly and map the outputs to $\mathbb{R}$ in such a way that the average correct output is $0$, then this class will have a low cross predictability. Obviously, we could still initialize a neural net to encode the function with the correct ordering of digits. However, if the net is initialized in a way that does not encode the digit's meanings, then deep learning will have difficulties learning this function comparable to its problems learning parity. Note that one can sort out which digit is which by taking enough samples where the expression is correct and the last digit of the sum is left, using them to derive linear equation in the digits $\pmod{n}$, and solving for the digits.

We believe that if the input contained the entire alleged sum, then deep learning with a random initialization would also be unable to learn to determine whether or not the sum was correct. However, in order to train it, one would have to give it correct expressions far more often than would arise if it was given random inputs drawn from a probability distribution that was independent of the digits' meanings. As such, our notion of cross predictability does not apply in this case, and the techniques we use in this paper do not work for the version where the entire alleged sum is provided. The techniques instead apply to the above version.\\

\noindent
{\bf Connectivity and community detection.} Another example of a problem that we believe deep learning would have trouble with consists of determining whether or not some graphs are connected. This could be difficult because it is a global property of the graph, and there is not necessarily any function of a small number of edges that is correlated with it. Of course, that depends on how the graphs are generated. In order to make it difficult, we define the following probability distribution for random graphs.

\begin{definition}
Given $n,m,r>0$, let $AER(n,m,r)$ be the probability distribution of $n$-vertex graphs generated by the following procedure. First of all, independently add an edge between each pair of vertices with probability $m/n$ (i.e., start with an Erd\H{o}s-R\'enyi random graph). Then, randomly select a cycle of length less than $r$ and delete one of its edges at random. Repeat this until there are no longer any cycles of length less than $r$.
\end{definition}

Now, we believe that deep learning with a random initialization will not be able to learn to distinguish a graph drawn from $AER(n,10\ln(n),\sqrt{\ln(n)})$ from a pair of graphs drawn from $AER(n/2,10\ln(n),\sqrt{\ln(n)})$, provided the vertices are randomly relabeled in the latter case. That is, deep learning will not distinguish between a patching of two such random graphs (on half of the vertices) versus a single such graph (on all vertices). Note that a simple depth-first search algorithm would learn the function in poly-time. More generally, we believe that deep learning would not solve community detection on such variants of random graph models\footnote{It would be interesting to investigate the approach of \cite{bruna_cd} on such models.} (with edges allowed between the clusters as in a stochastic block model with similar loop pruning), as connectivity v.s.\ disconnectivity is an extreme case of community detection.

The key issue is that no subgraph induced by fewer than $\sqrt{\ln(n)}$ vertices provides significant information on which of these cases apply. Generally, the function computed by a node in the net can be expressed as a linear combination of some expressions in small numbers of inputs and an expression that is independent of all small sets of inputs. The former cannot possibly be significantly correlated with the desired output, while the later will tend to be uncorrelated with any specified function with high probability. As such, we believe that the neural net would fail to have any nodes that were meaningfully correlated with the output, or any edges that would significantly alter its accuracy if their weights were changed. Thus, the net would have no clear way to improve.

\section{Proof techniques}  
Our main approach to showing the failure of an algorithm (e.g., SGD) using data from a test model (e.g, parities) with limited resources  (e.g., samples and memory) for a desired task (e.g., non-trivial accuracy of prediction), will be to establish an {\it indistinguishable to null condition (INC)}, namely, a condition on the resources that implies failure to statistically distinguish the trace of the algorithm on the test model from a null model, where the null model fails to provide the desired performance for trivial reasons. The INC is obtained by manipulating information measures, bounding the total variation distance of the two posterior measures between the test and null models. The failure of achieving the desired algorithmic performance on the test model is then a consequence of the INC, either by converse arguments -- if one could achieve the claimed  performance, one would be able to use the  performance gap to distinguish the null and test models  and thus contradict the INC -- or directly using the total variation distance between the two probability distributions to bound the difference in the probabilities that the nets drawn from those distributions compute the function correctly (and we know that it fails to do so on the null model).


With more details:
\begin{itemize}
\item Let $D_1$ be the  distribution of the data for the parity learning model, i.e., i.i.d.\ samples with labels from the parity model in dimension $n$;
\item Let $R=(R_1,R_2)$ be the resource in question, i.e., the number $R_1$ of edge weights of poly$(n)$ memory that are updated and the number of steps $R_2$ of the algorithm;
\item Let $A$ be the SGD (or coordinate descent) algorithm used with a constraint $C$ on the resource $R$; 
\item Let $T$ be the task, i.e, achieving an accuracy of $1/2 + \Omega_n(1)$ on a random input.
\end{itemize}

\noindent
Our program then runs as follows:
\begin{enumerate}
    \item Chose $D_0$ as the null distribution that generates i.i.d.\ pure noise labels, such that the task $T$ is obviously not achievable for $D_0$.
\item Find a INC on $R$, i.e., a constraint $C$ on $R$ such that the trace of the algorithm $A$ is indistinguishable under $D_1$ and $D_0$; to show this,
\begin{enumerate}
    \item show that the total variation distance between the posterior distribution of the trace of $A$ under $D_0$ and $D_1$ vanishes if the INC holds; to obtain this, 
    \item show that any $f$-mutual information between the algorithm's trace and the model hypotheses $D_0$ or $D_1$ (chosen equiprobably) vanishes.
\end{enumerate}
\item Conclude that the INC on $R$ prohibits the achievement of $T$ on the test model $D_0$, either by contradiction  as one could use $T$ to distinguish between $D_1$ and $D_0$ if only the latter fails at $T$ or using the fact that for any event $\mathrm{Success}$ and any random variables $Y(D_i)$ that depend on data drawn from $D_i$ (and represent for example the algorithms outputs), we have $\pp\{Y(D_1) \in \mathrm{Success} \} \le \pp\{Y(D_0) \in \mathrm{Success} \}+TV(D_0,D_1) = 1/2 + TV(D_0,D_1)$.
\end{enumerate}
Most of the work then lies in part 2(a)-(b), which consist in manipulating information measures to obtain the desired conclusion. In particular, the Chi-squared mutual information will be convenient for us, as its ``quadratic'' form will allow us to bring  the cross-predictability as an upper-bound, which is then easy to evaluate and is small for the parity model. This is carried out in Section \ref{sla} in the general context of so-called ``sequential learning algorithms'', and then applied to SGD with bounded memory (or coordinate descent) in Section \ref{results}. For Theorem \ref{thm2}, one needs also to take into account the fact that information can be carried in the  pattern of {\it which} weights can be updated, and these are taken into account with a proper SLA implementation, with Theorem \ref{thm2} concluding from a contradiction argument as discussed in step 3.\ above.

In the case of noisy GD (Theorems \ref{thm2} and \ref{thm2'}), the program is more direct from step 2, and runs with the following specifications. When computing the full gradient, the losses with respect to different inputs mostly cancel out, which makes the gradient updates reasonably  small, and a small amount of noise suffices to cover it. In this case, working\footnote{Similar results should hold for non-Gaussian distributions; the Gaussian case is more convenient.} with Gaussian noise allows us to bound the total variation distance in terms of the $\ell_2$ distance of the weight updates between the test vs.\ null models. 
We then to bound that
$\ell_2$ distance in terms of the gradient norm using classical orthogonality properties of the parity (Fourier-Walsh) basis. 

In the case of the failure of SGD under noisy initialization and updates (Theorem \ref{thm3}), we rely on a more sophisticated version of the above program. We use again a step used for GD that consists in showing that the average value of any function on samples generated by a random parity function will be approximately the same as the average value of the function on true random samples.\footnote{This gives also a variant of the result in \cite{ohad} applying with 1-Lipschitz loss function.} This is essentially a consequence of the low cross-predictability of parities, which can be proved more directly in this context. Most of the work consist then is using this to show that if we draw a set of weights in $\mathbb{R}^m$ from a sufficiently noisy probability distribution and then perturb it slightly in a manner dependent on a sample generated by a random parity function, the probability distribution of the result is essentially indistinguishable from what it would be if the samples were truly random. Then, we argue that if we do this repeatedly and add in some extra noise after each step, the probability distribution stays noisy enough that the previous result continues to apply, with the result that the final probability distribution when this is done using samples generated by a random parity function is similar to the final probability distribution using true random samples. After that, we use that result to show that the probability distribution of the weights in a neural net trained by noisy stochastic gradient descent on a random parity function is indistinguishable from the the probability distribution of the weights in a neural net trained by noisy SGD (NSGD) on random samples, which represent most of the work. Finally, we conclude that a neural net trained by NSGD on a random parity function will fail to learn the function (step 3).

\section{Cross-predictability and sequential learning}\label{sla}

\subsection{Cross-predictability}

We denote by $\X$ the domain of the data (e.g., Boolean vectors, matrices of pixels) and by $\Y$ the domain of the labels; for simplicity, we assume that $\Y=\{-1,+1\}$.
A hypothesis is a function $f: \X \to \Y$ that labels data points in $\X$ with elements in $\Y$. 
We define $\F:=\Y^\X$.

\begin{definition}
Let $P_\X$ be a probability measure on $\X$, and $P_{\F}$ be a probability measure on $\F$ (the set of functions from $\X$ to $\Y$). Define the cross-predictability of $P_{\F}$ with respect to $P_\X$  by 
\begin{align}
\mathrm{Pred}( P_\X, P_{\F} ) = \E_{F,F'} (\E_{X} F(X) F'(X))^2, \quad (X,F,F') \sim P_\X \times  P_{\F} \times P_{\F}.
\end{align}
\end{definition}
Note that 
\begin{align}
\E_{X} F(X) F(X') &= \pp_{X}\{F(X)=F'(X)\} -  \pp_{X}\{F(X) \neq F'(X)\}\\
&= 2\pp_{X}\{F(X)=F'(X)\} -  1
\end{align}
and we also have
\begin{align}
\mathrm{Pred}( P_\X, P_{\F} ) = \E_{X,X'} (\E_{F} F(X) F(X'))^2.
\end{align}
Therefore a low cross-predictability can be interpreted as having a low correlation between two sampled  functions on a sampled input, or, as a low correlation between two sampled input on a sampled function. This suggests that two samples - such as those of two consecutive steps of SGD - may not have common information about a typical function being learned. 

We discuss in Remark \ref{sq_disc} the analogy and difference between the cross-predictability and the statistical dimension. 
We next cover some examples. 

\begin{example} 
If $P_\X=\delta_x$, $x \in \X$, then $\mathrm{Pred}( P_\X, P_{\F} )=1$ no matter what $P_{\F}$ is. That is, if the world produces always the same data point, then any labelling function has maximal cross-predictability.
\end{example}
\begin{example} 
If $P_{\F}$ is the uniform probability measure on $\F$, then $\mathrm{Pred}( P_\X, P_{\F} )=\|P_\X \|_2^2$ no matter what $P_{\X}$ is. In fact, 
\begin{align}
\mathrm{Pred}( P_\X, P_{\F} ) &= \E_{F,F'} (\E_{X \sim P_\X} F(X) F'(X))^2\\
&=(1/|\F|)^2 \sum_{f,f'} \sum_{x,x'} f(x)f'(x) f(x')f'(x') P_\X(x)P_\X(x')\\
&=(1/|\F|)^2  \sum_{x,x'}  (\sum_{f} f(x)f(x'))^2 P_\X(x)P_\X(x')\\
&=(1/|\F|)^2  \sum_{x}  (\sum_{f} f^2(x))^2 P_\X^2(x)\\
&= \sum_{x} P_\X^2(x).
\end{align}
In particular, $\mathrm{Pred}( P_\X, P_{\F} )=1/|\X|$ if $P_{\F}$ and $P_\X$ are uniform. That is, if the labelling function is ``completely random,'' then the cross-predictability depends on how ``random'' the input distribution is, measured by the $L_2$ norm of $P_\X$, also called the collision entropy.
\end{example}
\begin{example} 
Let $\X=\B^n$, and for $s \in [n]$, define $f_s: \B^n \to \B$ by $f_S(x)=\prod_{i \in  S} x_i$. 
Let $P_{\F}=P_n$ be the uniform probability measure on $\{f_s\}_{s \in \B^n}$, and let $U_n$ be the uniform probability measure on $\B^n$. Then $\mathrm{Pred}(U_n, P_n )= 2^{-n}$. In fact, $\E_{S,T} f_S(X) f_T(X) = \1(S=T)$, thus $\mathrm{Pred}(U_n, P_n)= \pp_{S,T}\{S=T\}=2^{-n}$.
\end{example}
Note that uniformly random parity functions have the same cross-predictability as uniformly random generic functions, with respect to uniform inputs. We will crucially exploit this property to prove the forthcoming results. We obtain in fact generalizations of two our results to other function distributions having cross-predictability that scales super-polynomially.

\subsection{Learning from a bit}
We now consider the following setup:
\begin{align}
&(X,F) \sim P_\X \times P_\F  \label{r1} \\
&Y=F(X) \text{ (denote by $P_\Y$ the marginal of $Y$)}  \label{r2} \\
&W=g(X,Y) \text{ where $g: \X \times \Y \to \B$}  \label{r3} \\
&(\tX,\tY) \sim P_\X \times U_\Y \text{ (independent of $(X,F)$)}
\end{align}
That is, a random input $X$ and a random hypothesis $F$ are drawn from the working model, leading to an output label $Y$. 
We store a bit $W$ after observing the labelled pair $(X,Y)$. 
We are interested in estimating how much information can this bit contain about $F$, no matter how ``good'' the function $g$ is. 
We start by measuring the information using the variance of the MSE or  Chi-squared mutual information\footnote{The Chi-squared mutual information should normalize this expression with respect to the variance of $W$ for non equiprobable random variables.}, i.e., 
\begin{align}
I_2(W;F)=\Var \E (W|F) 
\end{align}
which gives a measure on how random $W$ is given $F$.
We provide below a bound in terms of the cross-predictability of $P_\F$ with respect to $P_\X$, and the marginal probability that $g$ takes value 1 on two independent inputs, which is a ``inherent bias'' of $g$.

The Chi-squared is convenient to analyze and is stronger than the classical mutual information, which is itself stronger than the squared total-variation distance by Pinsker's inequality. More precisely\footnote{See for example \cite{boix} for details on these inequalities.}, for an equiprobable $W$, 
\begin{align}
TV(W;F) \lesssim I(W;F)^{1/2} \le I_2(W;F)^{1/2}. 
\end{align}
Here we will need to obtain such inequalities for arbitrary marginal distributions of $W$ and in a self-contain series of lemmas. We then bound the latter with the cross-predictability which allows us to bound the error probability of the hypothesis test deciding whether $W$ is dependent on $F$ or not, which we later use in a more general framework where $W$ relates to the updated weights of the descent algorithm. We will next derive the bounds that are  needed.\footnote{These bounds could be slightly tightened but are largely sufficient for our purpose.}

\begin{lemma}\label{lemma_pred}
\begin{align}
 &\Var \E (g(X,Y)|F) \le 
 \E_F (\pp_X(g(X,F(X))=1)-\pp_{\tX,\tY}(g(\tX,\tY)=1))^2 \\ 
 &\le
 \min_{i \in \{0,1\}}   \pp\{ g(\tX,\tY) =i\} \sqrt{\mathrm{Pred}( P_\X, P_{\F} )}
\end{align}
\end{lemma}

\begin{proof}
Note that
\begin{align}
\Var \E (W|F) &=  \E_F (   \pp\{W=1|F\} - \pp\{W=1\}  )^2\\
&\le  \E_F (   \pp\{W=1|F\} -c  )^2
\end{align}
for any $c  \in \mR$.
Moreover,
\begin{align}
\pp\{W=1|F=f \} &= \sum_{x} \pp\{W=1|F=f, X=x \} P_\X(x)\\
&=\sum_{x,y} \pp\{W=1|X=x, Y=y \} P_\X(x) \1(f(x)=y) .
\end{align}
Pick now
\begin{align}
c:=\sum_{x,y} \pp\{W=1|X=x,Y=y \} P_\X(x) U_{\Y}(y) \label{py}
\end{align}

Therefore, 
\begin{align}
\pp\{W=1|F=f \} - c &=\sum_{x,y} A_g(x,y) B_f(x,y) =: \langle A_g,B_f \rangle
\end{align}
where 
\begin{align}
A_g(x,y):&=\pp\{W=1|X=x, Y=y \} \sqrt{P_{\X}(x)U_{\Y}(y)} \\ &= \pp\{g(X,Y)=1|X=x, Y=y \} \sqrt{P_{\X}(x)U_{\Y}(y)} \\
B_f(x,y):&=\frac{\1(f(x)=y) - U_{\Y}(y)}{U_{\Y}(y)} \sqrt{P_{\X}(x)U_{\Y}(y)}.
\end{align}
We have
\begin{align}
 \langle A_g,B_F \rangle^2 =  \langle A_g,B_F \rangle \langle B_F,A_g \rangle =   \langle A_g^{\otimes 2}, B_F^{\otimes 2} \rangle
\end{align}
and therefore 
\begin{align}
\E_F \langle A_g,B_F \rangle^2  &=   \langle A_g^{\otimes 2}, \E_F  B_F^{\otimes 2} \rangle \\
&\le \| A_g^{\otimes 2} \|_2 \| \E_F  B_F^{\otimes 2} \|_2.
\end{align}
Moreover, 
\begin{align}
 \| A_g^{\otimes 2} \|_2 &=  \| A_g \|_2^2 \\
&= \sum_{x,y} \pp\{W=1|X=x, Y=y \}^2 P_{\X}(x)U_{\Y}(y)\\
&\le \sum_{x,y} \pp\{W=1|X=x, Y=y \}  P_{\X}(x)U_{\Y}(y)\\
&= \pp\{ W(\tX,\tY)=1 \} 
\end{align}
and
\begin{align}
\| \E_F  B_F^{\otimes 2} \|_2 &= \left(\sum_{x,y,x',y'}  (\sum_{f} B_f(x,y) B_f(x',y') P_\F(f))^2 \right)^{1/2}\\
&=  \left( \E_{F,F'}  \langle B_F, B_{F'} \rangle^2  \right)^{1/2}.
\end{align}
Moreover, 
\begin{align}
\langle B_f, B_{f'} \rangle &= \sum_{x,y}  \frac{\1(f(x)=y) - U_{\Y}(y)}{U_{\Y}(y)} \frac{\1(f'(x)=y) -U_{\Y}(y)}{U_{\Y}(y)}  P_\X(x) U_{\Y}(y) \\
& =(1/2) \sum_{x,y}  (2 \1(f(x)=y) - 1)  (2 \1(f'(x)=y) - 1)  P_\X(x) \\
&=\E_X f(X)f'(X) .
\end{align}
Therefore, 
\begin{align}
 \Var \pp \{W=1| F\} \le \pp\{ \tW=1 \} \sqrt{\mathrm{Pred}( P_\X, P_{\F} )}.
\end{align}
The same expansion holds with $ \Var \pp \{W=1| F\}=\Var \pp \{W=0| F\} \le \pp\{ \tW=0 \} \sqrt{\mathrm{Pred}( P_\X, P_{\F} )}$.

\end{proof}

Consider now the new setup where $g$ is valued in $[m]$ instead of $\{0,1\}$:
\begin{align}
&(X,F) \sim P_\X \times P_\F \label{s1} \\
&Y=F(X)  \label{s2} \\
&W=g(X,Y) \text{ where $g: \B^n \times \Y \to [m]$}. \label{s3}
\end{align}
We have the following theorem. 
\begin{theorem} \label{corol_unif2}
\[E_F\|P_{W|F}-P_W\|_2^2\le\sqrt{\mathrm{Pred}( P_\X, P_{\F} )}\]
\end{theorem}

\begin{proof}
From Lemma \ref{lemma_pred}, for any $i \in [m]$, 
\begin{align}
\Var \pp \{W=i|F\} \le \pp\{ g(\tX,\tY)=i\} \sqrt{\mathrm{Pred}( P_\X, P_{\F} )},
\end{align}
therefore, 
\begin{align}
E_F\|P_{W|F}-P_W\|_2^2  &= \sum_{i \in [m]} \sum_{f \in F} \pp\{F=f\} (\pp\{W=i|F=f\} -\pp\{W=i\})^2 \\
&\le\sum_{i \in [m]} \pp\{ g(\tX,\tY)=i\} \sqrt{\mathrm{Pred}( P_\X, P_{\F} )}\\\
& =\sqrt{\mathrm{Pred}( P_\X, P_{\F} )}.
\end{align}
\end{proof}

\begin{corollary}\label{thm_pred}
\begin{align}
\| P_{W,F} - P_W P_F \|_2^2  &\le \| P_\F \|_\infty   \sqrt{\mathrm{Pred}( P_\X, P_{\F} )}.
\end{align}
\end{corollary}
We next specialize the bound in Theorem \ref{thm_pred} to the case of uniform parity functions on uniform inputs, adding a bound on the $L_1$ norm due to Cauchy-Schwarz.

\begin{corollary}\label{}
Let $m,n >0$. If we consider the setup of \eqref{s1},\eqref{s2},\eqref{s3} for the case where $P_\F=P_n$, the uniform probability measure on parity functions, and $P_\X=U_n$, the uniform probability measure on $\B^n$, then 
\begin{align}
& \| P_{W,F} - P_W P_F \|_2^2 \le 2^{-(3/2)n},\\
& \| P_{W,F} - P_W P_F \|_1 \le \sqrt{m}2^{-n/4}.
\end{align}
\end{corollary}
In short, the value of $W$ will not provide significant amounts of information on $F$ unless its number of possible values $m$ is exponentially large.

\begin{corollary}\label{corol_unif}
Consider the same setup as in previous corollary, with in addition $(\tX,\tY)$ independent of $(X,F)$ such that $(\tX,\tY) \sim P_\X \times U_\Y$ where $U_\Y$ is the uniform distribution on $\Y$, and $\tW=g(\tX,\tY)$. Then,
\[\sum_{i\in[m]} \sum_{s\subseteq [n]} (P[W=i|f=p_s]-P[\tW=i])^2\le 2^{n/2}.\]
\end{corollary}

\begin{proof}
In the case where $P_{\F}=P_n$, taking the previous corollary and multiplying both sides by $2^{2n}$ yields
\[\sum_{i\in[m]} \sum_{s\subseteq [n]} (P[W=i|f=p_s]-P[W=i])^2\le 2^{n/2}.\]
Furthermore, the probability distribution of $(X,Y)$ and the probability distribution of $(\tX,\tY)$ are both $U_{n+1}$ so $P[\tW=i]=P[W=i]$ for all $i$. Thus,
\begin{align}\sum_{i\in[m]} \sum_{s\subseteq [n]} (P[W=i|f=p_s]-P[\tW=i])^2\le 2^{n/2}. \label{lastw} \end{align}
\end{proof}
Notice that for fixed values of $P_{\X}$ and $g$, changing the value of $P_{\F}$ does not change the value of $P[W=i|f=p_s]$ for any $i$ and $s$. Therefore, inequality \eqref{lastw} holds for any choice of $P_{\F}$, and we also have the  following.

\begin{corollary}
Consider the general setup of \eqref{s1},\eqref{s2},\eqref{s3} with $P_\X=U_n$, and $(\tX,\tY)$ independent of $(X,F)$ such that $(\tX,\tY) \sim P_\X \times U_\Y$, $\tW=g(\tX,\tY)$.
Then,
\[\sum_{i\in[m]} \sum_{s\subseteq [n]} (P[W=i|f=p_s]-P[\tW=i])^2\le 2^{n/2}.\]
\end{corollary}

%

\subsection{Sequential learning algorithm}
Next, we would like to analyze the effectiveness of an algorithm that repeatedly receives an ordered pair, $(X,F(X))$, records some amount of information about that pair, and then forgets it. To formalize this concept, we define the following.


\begin{definition}   
A sequential learning algorithm $A$ on $(\mathcal{Z}, \mathcal{W})$ is an algorithm that for an input of the form $(Z,(W_1,...,W_{t-1}))$ in $\mathcal{Z} \times \mathcal{W}^{t-1}$ produces an output $A(Z,(W_1,...,W_{t-1}))$ valued in $\mathcal{W}$. Given a probability distribution $D$ on $\mathcal{Z}$, a sequential learning algorithm $A$ on $(\mathcal{Z}, \mathcal{W})$, and $T\ge 1$, a $T$-trace of $A$ for $D$ is a series of pairs $((Z_1, W_1), ...,(Z_T, W_T))$ such that for each $i \in [T]$, $Z_i\sim D$ independently of $(Z_1,Z_2,...,Z_{i-1})$ and $W_i=A(Z_i,(W_1,W_2,...,W_{i-1}))$.
\end{definition}

If $|\mathcal{W}|$ is sufficiently small relative to $\mathrm{Pred}( P_\X, P_{\F} )$, then a sequential learning algorithm that outputs elements of $\mathcal{W}$ will be unable to effectively distinguish between a random function from $P_{\F}$ and a true random function in the following sense.


\begin{theorem} \label{SLAfail} 
Let $n>0$, $A$ be a sequential learning algorithm on $(\B^{n+1},\mathcal{W})$, $P_{\X}$ be the uniform distribution on $\B^n$, and $P_{\F}$ be a probability distribution on functions from $\B^n$ to $\B$. Let $\star$ be the probability distribution of $(X,F(X))$ when $F\sim P_{\F}$ and $X\sim P_{\X}$. Also, for each $f:\B^n\to\B$, let let $\rho_f$ be the probability distribution of $(X,f(X))$ when $X\sim P_{\X}$. Next, let $P_{\mathcal{Z}}$ be a probability distribution on $\B^{n+1}$ that is chosen by means of the following procedure: with probability $1/2$, set $P_{\mathcal{Z}}=\star$, otherwise draw $F\sim P_{\F}$ and set $P_{\mathcal{Z}}=\rho_F$. If $|\mathcal{W}|\le 1/\sqrt[24]{\mathrm{Pred}( P_\X, P_{\F} )}$, $m$ is a positive integer with $m<1/\sqrt[24]{\mathrm{Pred}( P_\X, P_{\F} )}$, and $((Z_1,W_1),...,(Z_m,W_m))$ is a $m$-trace of $A$ for $P_{\mathcal{Z}}$, then
\begin{align}
\| P_{W^m|P_{\mathcal{Z}}=\star} - P_{W^m|P_{\mathcal{Z}}\ne \star} \|_{1} =  O(\sqrt[24]{\mathrm{Pred}( P_\X, P_{\F} )}).
\end{align}
\end{theorem}

%
%

\begin{proof} 
First of all, let $q=\sqrt[24]{\mathrm{Pred}( P_\X, P_{\F} )}$ and $F'\sim P_{\F}$. Note that by the triangular inequality,
\begin{align*}
&\| P_{W^m|P_{\mathcal{Z}}=\star} - P_{W^m|P_{\mathcal{Z}}\ne \star} \|_{1} \\
&=\sum_{w_1,...,w_m\in\mathcal{W}} |P[W^m=w^m|P_{\mathcal{Z}}  \ne \star]-P[W^m=w^m|P_{\mathcal{Z}}=\star]| \\
&\le \sum_{f:\B^n\to\B}P[F=f]\sum_{w^m\in\mathcal{W}^m} |P[W^m=w^m|P_{\mathcal{Z}}=\rho_s]-P[W^m=w^m|P_{\mathcal{Z}}=\star]|
\end{align*}
and we will bound the last term by $O(q)$.

We need to prove that $P[W^m=w^m|P_{\mathcal{Z}}=\rho_f]\approx P[W^m=w^m|P_{\mathcal{Z}}=\star]$ most of the time. In order to do that, we will use the fact that
\[\frac{P[W^m=w^m|P_{\mathcal{Z}}=\rho_f]}{P[W^m=w^m|P_{\mathcal{Z}}=\star]}=\prod_{i=1}^m \frac{P[W_i=w_i|W^{i-1}=w^{i-1},P_{\mathcal{Z}}=\rho_f]}{P[W_i=w_i|W^{i-1}=w^{i-1},P_{\mathcal{Z}}=\star]}\]
So, as long as $P[W_i=w_i|W^{i-1}=w^{i-1},P_{\mathcal{Z}}=\rho_f]\approx P[W_i=w_i|W^{i-1}=w^{i-1},P_{\mathcal{Z}}=\star]$ and $P[W_i=w_i|W^{i-1}=w^{i-1},P_{\mathcal{Z}}=\star]$ is reasonably large for all $i$, this must hold for the values of $w^m$ and $f$ in question. As such, we plan to define a good value for $(w^m,f)$ to be one for which this holds, and then prove that the set of good values has high probability measure.

First, call a sequence $w^m\in \mathcal{W}^m$ {\it typical} if for each $1\le i\le m$, we have that 
\[t(w^i):= P[W_i=w_i|W^{i-1}=w^{i-1},P_{\mathcal{Z}}=\star]\ge q^3,\]
and denote by $\mathcal{T}$ the set of typical sequences
\begin{align}
\mathcal{T} &:= \{ w^m: \forall i \in [m], t(w^i)\ge q^3 \}.
\end{align}

We have
\begin{align}
1 &=  \pp\{ W^m \in \mathcal{T} | P_{\mathcal{Z}}=\star \}  +   \pp\{ W^m \notin \mathcal{T} | P_{\mathcal{Z}}=\star\}\\
&\le \pp\{ W^m \in \mathcal{T}| P_{\mathcal{Z}}=\star \}  + \sum_{i=1}^m  \pp\{  t(W^i) <  q^3| P_{\mathcal{Z}}=\star \} \\
&\le \pp\{ W^m \in \mathcal{T}| P_{\mathcal{Z}}=\star \}  +  m q^3 |\mathcal{W}|.
\end{align}
Thus
\begin{align}
\pp\{ W^m \in \mathcal{T}| P_{\mathcal{Z}}=\star \}&\ge 1-  m q^3 |\mathcal{W}| \ge 1-q.
\end{align}

Next, call an ordered pair of a sequence $w^m\in \mathcal{W}^m$ and an $f:\B^n\to\B$ {\it good} if $w^m$ is typical and
\begin{align}
\left| \frac{P[W_i=w_i|W^{i-1}=w^{i-1},P_{\mathcal{Z}}=\rho_f]}{P[W_i=w_i|W^{i-1}=w^{i-1},P_{\mathcal{Z}}=\star]}-1\right| 
\le q^2  , \quad  \forall i  \in [m],
\end{align}
and denote by $\mathcal{G}$ the set of good pairs. A pair which is not good is called bad.

Note that for any $i$ and any $w_1,...,w_{i-1}\in \mathcal{W}$, there exists a function $g_{w_1,...,w_{i-1}}$ such that $W_i=g_{w_1,...,w_{i-1}}(Z_i)$. So, theorem \ref{corol_unif2} implies that
\begin{align}
&\sum_{w_i\in\mathcal{W}}\sum_{f:\B^n\to\B} P[F'=f](P[W_i=w_i|W^{i-1}=w^{i-1},P_{\mathcal{Z}}=\rho_f]-P[W_i=w_i|W^{i-1}=w^{i-1},P_{\mathcal{Z}}=\star])^2  \\
&=\sum_{w_i\in\mathcal{W}}\sum_{f:\B^n\to\B} P[F'=f] (P[g_{w_1,...,w_{i-1}}(Z_i)=w_i|P_{\mathcal{Z}}=\rho_f]-P[g_{w_1,...,w_{i-1}}(Z_i)=w_i|P_{\mathcal{Z}}=\star])^2  \\
&\le q^{12}
\end{align}

Also, given any $w^m$ and $f:\B^{n}\to\B$ such that $w^m$ is typical but $w^m$ and $f$ are not good, there must exist $1\le i\le m$ such that 
\begin{align}
r(w^i,f)&:=|P[W_i=w_i|W^{i-1}=w^{i-1},P_{\mathcal{Z}}=\rho_f]-P[W_i=w_i|W^{i-1}=w^{i-1},P_{\mathcal{Z}}=\star]|\\
&\ge q^5.
\end{align}
Thus, for $w^m \in \mathcal{T}$
\begin{align}
\sum_{f: (w^m,f) \notin  \mathcal{G} } P[F'=f] &=\pp\{ (w^m,F') \notin  \mathcal{G} \} \\
&\le \pp\{ \exists i \in [m]: r(w^i,F') \ge q^5 \}\\
&\le \sum_{i=1}^m \sum_{f: r(w^i,f) \ge q^5} P[F'=f]\\
&\le  q^{-10} \sum_{i=1}^m \sum_{f:\B^n\to\B} P[F'=f]\cdot r(w^i,f)^2   \\
&\le q^{-10} \sum_{i=1}^m \sum_{w_i' \in \mathcal{W}}  \sum_{f:\B^n\to \B} P[F'=f]\cdot r((w_i',w^{i-1}),f)^2   \\
&\le q^{-10} m \cdot q^{12}  \\
\end{align}
This means that for a given typical $w^m$, the probability that $w^m$ and $F'$ are not good is at most $m q^2\le q$.

Therefore, if $P_{\mathcal{Z}}=\star$, the probability that $W^m$ is typical but $W^m$ and $F'$ is not good is at most $q$; in fact:
\begin{align}
& \pp\{ W^m \in \mathcal{T} , (W^m,F') \notin  \mathcal{G} | P_{\mathcal{Z}}=\star \}\\
& = \sum_{f,  w^m \in \mathcal{T} : (w^m,s) \notin  \mathcal{G} } \pp\{F'=f\}\cdot \pp\{W^m=w^m | P_{\mathcal{Z}}=\star \}  \\
& =\sum_{ w^m \in \mathcal{T}} \pp\{W^m=w^m | P_{\mathcal{Z}}=\star \} \sum_{f : (w^m,s) \notin  \mathcal{G} }  \pp\{F'=f\}  \\
& \le q \sum_{ w^m \in \mathcal{T}} \pp\{W^m=w^m | P_{\mathcal{Z}}=\star \} \\
& \le q.
\end{align}

We already knew that $W^m$ is typical with probability $1-q$ under these circumstances, so $W^m$ and $S$ is good with probability at least $1-2 q$ since
\begin{align}
&1-q  \le \pp\{ W^m \in \mathcal{T}  | P_{\mathcal{Z}}=\star \}\\
& = \pp\{ W^m \in \mathcal{T} , (W^m,F') \in  \mathcal{G} | P_{\mathcal{Z}}=\star \}  +   \pp\{ W^m \in \mathcal{T} , (W^m,F') \notin  \mathcal{G} | P_{\mathcal{Z}}=\star \}  \\
&\le \pp\{ (W^m,F') \in  \mathcal{G} | P_{\mathcal{Z}}=\star \} + q.
\end{align}

Next, recall that 
\[\frac{P[W^m=w^m|P_{\mathcal{Z}}=\rho_f]}{P[W^m=w^m|P_{\mathcal{Z}}=\star]}=\prod_{i=1}^m \frac{P[W_i=w_i|W^{i-1}=w^{i-1},P_{\mathcal{Z}}=\rho_f]}{P[W_i=w_i|W^{i-1}=w^{i-1},P_{\mathcal{Z}}=\star]}\]
So, if $w^m$ and $f$ is good (and thus each term in the above product is within $q^2$ of 1), we have 
\begin{align}
&\left| \frac{P[W^m=w^m|P_{\mathcal{Z}}=\rho_f]}{P[W^m=w^m|P_{\mathcal{Z}}=\star]}-1\right|  \le e^{q}-1 =O(q).
\end{align}
That implies that
\begin{align*}
&\sum_{(w^m,f) \in \mathcal{G}}P[F'=f]\cdot  |P[W^m=w^m|P_{\mathcal{Z}}=\rho_f]-P[W^m=w^m|P_{\mathcal{Z}}=\star]|\\
&\le \sum_{(w^m,f) \in \mathcal{G}} P[F'=f]\cdot O(q)\cdot P[W^m=w^m|P_{\mathcal{Z}}=\star]\\
&\le \sum_{w^m} O(q)\cdot P[W^m=w^m|P_{\mathcal{Z}}=\star]\\
&=O(q).
\end{align*}

Also,
\begin{align*}
&\sum_{(w^m,f) \notin \mathcal{G}} P[F'=f]\cdot (P[W^m=w^m|P_{\mathcal{Z}}=\rho_f]-P[W^m=w^m|P_{\mathcal{Z}}=\star])\\
&= P[(W^m,F') \notin \mathcal{G}|P_{\mathcal{Z}}\ne\star]-P[(W^m,F') \notin \mathcal{G}|P_{\mathcal{Z}}=\star]\\
&= P[(W^m,F') \in \mathcal{G}|P_{\mathcal{Z}}=\star]-P[W^m,F') \in \mathcal{G}|P_{\mathcal{Z}}\ne\star]\\
&= \sum_{(w^m,f) \in \mathcal{G}} P[F'=f]\cdot (P[W^m=w^m|P_{\mathcal{Z}}=\star]-P[W^m=w^m|P_{\mathcal{Z}}=\rho_f]) \\
&\le \sum_{(w^m,f) \in \mathcal{G}} P[F'=f]\cdot |P[W^m=w^m|P_{\mathcal{Z}}=\star]-P[W^m=w^m|P_{\mathcal{Z}}=\rho_f]| \\
&=O(q).
\end{align*}
That means that
\begin{align*}
&\sum_{(w^m,f) \notin \mathcal{G}} P[F'=f]\cdot  |P[W^m=w^m|P_{\mathcal{Z}}=\rho_f]-P[W^m=w^m|P_{\mathcal{Z}}=\star]|\\
&\le \sum_{(w^m,f) \notin \mathcal{G}} P[F'=f]\cdot  (P[W^m=w^m|P_{\mathcal{Z}}=\rho_f]+P[W^m=w^m|P_{\mathcal{Z}}=\star])\\
&= \sum_{(w^m,f) \notin \mathcal{G}} P[F'=f]\cdot 2P[W^m=w^m|P_{\mathcal{Z}}=\star]\\
&\qquad\qquad +\sum_{(w^m,f) \notin \mathcal{G}} P[F'=f]\cdot (P[W^m=w^m|P_{\mathcal{Z}}=\rho_f]-P[W^m=w^m|P_{\mathcal{Z}}=\star])\\
&=O(q).
\end{align*}
Therefore,
\begin{align}&\sum_{f:\B^n\to\B}\sum_{w^m\in\mathcal{W}} P[F'=f]\cdot |P[W^m=w^m|P_{\mathcal{Z}}=\rho_s]-P[W^m=w^m|P_{\mathcal{Z}}=\star]|=O(q),\end{align}
which gives the desired bound. 
\end{proof}

\begin{corollary} $\label{memLimit}$
Consider a data structure with a polynomial amount of memory that is divided into variables that are each $O(\log n)$ bits long, and define $m$, $\mathcal{Z}$, $\star$, and $P_{\mathcal{Z}}$ the same way as in Theorem \ref{SLAfail}. Also, let $A$ be an algorithm that takes the data structure's current value and an element of $\B^{n+1}$ as inputs and changes the values of at most $o(-\log(\mathrm{Pred}( P_\X, P_{\F} ))/\log(n))$ of the variables. If we draw $Z_1,...,Z_m$ independently from $P_{\mathcal{Z}}$ and then run the algorithm on each of them in sequence, then no matter how the data structure is initialized, it is impossible to determine whether or not $P_{\mathcal{Z}}=\star$ from the data structure's final value with accuracy greater than $1/2+O(\sqrt[24]{\mathrm{Pred}( P_\X, P_{\F} )})$.
\end{corollary}

\begin{proof}
Let $q=1/\sqrt[24]{\mathrm{Pred}( P_\X, P_{\F} )}$. Let $W_0$ be the initial state of the data structure's memory, and let $W_i=A(W_{i-1}, Z_i)$ for each $0<i\le m$. Next, for each such $i$, let $W'_i$ be the list of all variables that have different values in $W_i$ than in $W_{i-1}$, and their values in $W_i$. There are only polynomially many variables in memory, so it takes $O(\log(n))$ bits to specify one and $O(\log(n))$ bits to specify a value for that variable. $A$ only changes the values of $o(\log(q)/\log(n))$ variables at each timestep, so $W'_i$ will only ever list $o(\log(q)/\log(n))$ variables. That means that $W'_i$ can be specified with $o(\log(q))$ bits, and in particular that there exists some set $\mathcal{W}$ such that $W'_i$ will always be in $\mathcal{W}$ and $|\mathcal{W}|=2^{o(\log(q))}$. Also, note that we can determine the value of $W_i$ from the values of $W_{i-1}$ and $W'_i$, so we can reconstruct the value of $W_i$ from the values of $W'_1,W'_2,...,W'_i$. 

Now, let $A'$ be the algorithm that takes $(Z_t, (W'_1,...,W'_{t-1}))$ as input and does the following. First, it reconstructs $W_{t-1}$ from $(W'_1,...,W'_{t-1})$. Then, it computes $W_t$ by running $A$ on $W_{t-1}$ and $Z_t$. Finally, it determines the value of $W'_t$ by comparing $W_t$ to $W_{t-1}$ and returns it. This is an SLA, and $((Z_1,W'_1),...,(Z_m,W'_m))$ is an $m$-trace of $A'$ for $P_{\mathcal{Z}}$. So, by the theorem

\begin{align}&  \sum_{w_1,...,w_m} |P[W'_1=w_1,...,W'_m=w_m|P_{\mathcal{Z}}\ne \star ]-P[W'_1=w_1,...,W'_m=w_m|P_{\mathcal{Z}}=\star]|\\
&=O(1/q)\end{align}

Furthermore, since $W_m$ can be reconstructed from $(W'_1,...,W'_m)$, this implies that
\begin{align}
\sum_{w} |P[W_m=w|P_{\mathcal{Z}}\ne \star]-P[W_m=w|P_{\mathcal{Z}}=\star]|=O(1/q).
\end{align}
Finally, the probability of deciding correctly between the hypothesis $P_\Z=\star$ and $P_\Z\ne \star$ given the observation $W_m$ is at most 
\begin{align}
&1-\frac{1}{2}\sum_{w \in \mathcal{W} }  P[W_m=w|P_{\mathcal{Z}}= \star] \wedge  P[W_m=w|P_{\mathcal{Z}}\ne \star] \\
& = \frac{1}{2} + \frac{1}{4} \sum_{w \in \mathcal{W} } |P[W_m=w|P_{\mathcal{Z}}\ne \star]-P[W_m=w|P_{\mathcal{Z}}=\star]|\\
& = \frac{1}{2}  + O(1/q),
\end{align}
which implies the conclusion. 

\end{proof}


\begin{remark}
The theorem and its second corollary state that the algorithm can not determine whether or not $P_{\mathcal{Z}}=\star$. However, one could easily transform them into results showing that the algorithm can not effectively learn to compute $f$. More precisely, after running on $q/2$ pairs $(x,p_f(x))$, the algorithm will not be able to compute $p_s(x)$ with accuracy $1/2+\omega(1/\sqrt{q})$ with a probability of $\omega(1/q)$. If it could, then we could just train it on the first $m/2$ of the $Z_i$ and count how many of the next $m/2$ $Z_i$ it predicts the last bit of correctly. If $P_{\mathcal{Z}}=\star$, each of those predictions will be independently correct with probability $1/2$, so the total number it is right on will  differ from $m/4$ by $O(\sqrt{m})$ with high probability. However, if $P_{\mathcal{Z}}=\rho_f$ and the algorithm learns to compute $\rho_f$ with accuracy $1/2+\omega(1/\sqrt{q})$, then it will predict $m/4+\omega(\sqrt{m})$ of the last $m/2$ correctly with high probability. So, we could determine whether or not $P_{\mathcal{Z}}=\star$ with greater accuracy than the theorem allows by tracking the accuracy of the algorithm's predictions.
\end{remark}
\section{Negative results for deep learning}
\subsection{Neural nets and deep learning}\label{dl}

Before we can talk about the effectiveness of deep learning at learning these parity functions, we will have to establish some basic information about deep learning. First of all, in this paper we will be using the following definition for a neural net.

\begin{definition}\label{def_g}
A neural net is a pair of a function $f:\mathbb{R}\rightarrow \mathbb{R}$ and a weighted directed graph $G$ with some special vertices and the following properties. Frist of all, $G$ does not contain any cycles. Secondly, there exists $n>0$ such that $G$ has exactly $n+1$ vertices that have no edges ending at them, $v_0$, $v_1$,...,$v_n$. We will refer to $n$ as the input size, $v_0$ as the constant vertex and $v_1$, $v_2$,..., $v_n$ as the input vertices. Finally, there exists a vertex $v_{out}$ such that for any other vertex $v'$, there is a path from $v'$ to $v_{out}$ in $G$. We also use denote by $w(G)$ the weights on the edges of $G$. 
\end{definition}

\begin{definition}
Given a neural net $(f,G)$ with input size $n$, and $x\in\mathbb{R}^n$, the evaluation of $(f,G)$ at $x$, is the number computed by means of the following procedure.
\begin{enumerate}
\item Define $y\in \mathbb{R}^{|G|}$ where $|G|$ is the number of vertices in $G$, set $y_{v_0}=1$, and set $y_{v_i}=x_i$ for each $i$.

\item Find an ordering $v'_1,...,v'_m$ of the vertices in $G$ other than the constant vertex and input vertices such that for all $j>i$, there is not an edge from $v'_j$ to $v'_i$.

\item For each $1\le i\le m$, set 
\[y_{v'_i}=f\left(\sum_{v: (v,v'_i)\in E(G)} w_{v,v'_i} y_v\right)\]

\item Return $y_{v_{out}}$.
\end{enumerate}

We will write the evaluation of $(f,G)$ at $x$ as $eval_{(f,G)}(x)$.
\end{definition}

Generally, we want to find a neural net that computes a certain function, or at least a good approximation of that function. A reasonable approach to doing that is to start with some neural network and then attempt to adjust its weights until it computes a reasonable approximation of the desired function. A common way to do that is to define a loss function in terms of how much the network's outputs differ from the desired outputs, and then use gradient descent to try to adjust the weights. Of course, that may not be well defined if $f$ is not differentiable, it may not be possible to find weights for which the network approximates the desired function if it has too few vertices or is missing some key edges, and the gradient descent algorithm has an increased risk of getting stuck in a local minimum if $f$ has local minima. Allowing the loss function or the derivative of $f$ to take on arbitrarily large values under some circumstances can also cause problems, which can be mitigated by redefining our function to have input in $[0,1]^n$ and output in $[0,1]$ and then using an $f$ with output in $[0,1]$. We would also like to ensure that $f$ can take on values arbitrarily close to any desired output in that range. As such, we give the following criteria for a neural net to be considered well behaved.

\begin{definition}
Let $(f,G)$ be a neural net. Then $(f,G)$ is normal if it satisifes the following properties. $f$ must be a a smooth function, the derivative of $f$ must be positive everywhere, the derivative of $f$ must be bounded, it must be the case that $\lim_{x\to -\infty} f(x)=0$ and $\lim_{x\to \infty} f(x)=1$, and $G$ must have an edge from the constant vertex to every other vertex except the input vertices.
\end{definition}

More formally, given a target function $h$, a probability distribution $P_\mathcal{X}$ for the input of $h$, and a loss function $L:\mathbb{R}\rightarrow \mathbb{R}$, we would like to find a neural net $(f, G)$ that has a small value of 
\[E_{X\sim P_\mathcal{X}}[L(eval_{(f,G)}(X)-h(X))]\]
In order to do that, we could try starting with some neural net $(f,G_0)$, and then using the following algorithm to assign new weights to the graph's edges.

\vspace{1 cm}
\noindent
{\em GradientDescentStep(f, G, h, $P_\mathcal{X}$, L, $\gamma$, B)}:
\begin{enumerate}
\item For each $(v,v')\in E(G)$: \begin{enumerate}
\item Set 
\[w'_{v,v'}=w_{v,v'}-\gamma \frac{\partial E_{X\sim P_\mathcal{X}} [L(eval_{(f,G)}(X)-h(X))]}{\partial w_{v,v'}}\]

\item If $w'_{v,v'}<-B$, set $w'_{v,v'}=-B$.

\item If $w'_{v,v'}>B$, set $w'_{v,v'}=B$.

\end{enumerate}

\item Return the graph that is identical to $G$ except that its edge weight are given by the $w'$.
\end{enumerate}

\vspace{1 cm}
\noindent
{\em GradientDescentAlgorithm(f, G, h, $P_\mathcal{X}$, L, $\gamma$, B, t)}:
\begin{enumerate}
\item Set $G_0=G$.

\item If any of the edge weights in $G_0$ are less than $-B$, set all such weights to $-B$.

\item If any of the edge weights in $G_0$ are greater than $B$, set all such weights to $B$.

\item For each $0\le i<t$, set $G_{i+1}=GradientDescentStep(f, G_i, h, P_\mathcal{X}, L, \gamma, B)$.

\item Return $G_t$.
\end{enumerate}

The hope is that if we set $G'=GradientDescentAlgorithm(f, G, h, P_\mathcal{X}, L, \gamma, B, t)$ for a small enough $\gamma$ and large enough $t$ then $eval_{(f,G')}$ will be a good approximation of $h$. Of course, actually running this algorithm requires us to compute $eval_{(f,G)}(X)$ for every possible value of $X$ in every step, which is generally impractical. As a result, we are more likely to pick a single value of $X$ at each step, and adjust the net to give a better output on that input. More formally, we would use the following algorithm.

\vspace{1 cm}
\noindent
{\em SampleGradientDescentStep(f, G, $Y$, $X$, L, $\gamma$, B)}:
\begin{enumerate}
\item For each $(v,v')\in E(G)$: \begin{enumerate}
\item Set 
\[w'_{v,v'}=w_{v,v'}-\gamma \frac{\partial L(eval_{(f,G)}(X)-Y)}{\partial w_{v,v'}}\]

\item If $w'_{v,v'}<-B$, set $w'_{v,v'}=-B$.

\item If $w'_{v,v'}>B$, set $w'_{v,v'}=B$.

\end{enumerate}

\item Return the graph that is identical to $G$ except that its edge weights are given by the $w'$.
\end{enumerate}

\vspace{1 cm}
\noindent
{\em StochasticGradientDescentAlgorithm(f, G, $P_\mathcal{Z}$, L, $\gamma$, B, t)}:
\begin{enumerate}
\item Set $G_0=G$.

\item If any of the edge weights in $G_0$ are less than $-B$, set all such weights to $-B$.

\item If any of the edge weights in $G_0$ are greater than $B$, set all such weights to $B$.

\item For each $0\le i<t$:
\begin{enumerate}

\item Draw $(X_i, Y_i)\sim P_\mathcal{Z}$, independently of all previous values.

\item Set $G_{i+1}=SampleGradientDescentStep(f, G_i, Y_i, X_i, L, \gamma, B)$
\end{enumerate}

\item Return $G_t$.
\end{enumerate}

\subsection{Proof of Theorem  1}

One possible variant of this is to only adjust a few weights at each time step, such as the $k$ that would change the most or a random subset. However, any such algorithm cannot learn a random parity function in the following sense.

\begin{theorem}
Let $n>0$, $k=o(n/\log(n))$, and $(f,g)$ be a neural net of size polynomial in $n$ in which each edge weight is recorded using $O(\log n)$ bits. Also, let $\star$ be the uniform distribution on $\B^{n+1}$, and for each $s\subseteq [n]$, let $\rho_s$ be the probability distribution of $(X,p_s(X))$ when $X$ is chosen randomly from $\B^n$. Next, let $P_{\mathcal{Z}}$ be a probability distribution on $\B^{n+1}$ that is chosen by means of the following procedure. First, with probability $1/2$, set $P_{\mathcal{Z}}=\star$. Otherwise, select a random $S\subseteq[n]$ and set $P_{\mathcal{Z}}=\rho_S$. Then, let $A$ be an algorithm that draws a random element from $P_{\mathcal{Z}}$ in each time step and changes at most $k$ of the weights of $g$ in response to the sample and its current values. If $A$ is run for less than $2^{n/24}$ time steps, then it is impossible to determine whether or not $P_{\mathcal{Z}}=\star$ from the resulting neural net with accuracy greater than $1/2+O(2^{-n/24})$.
\end{theorem}

\begin{proof}
This follows immediately from corollary \ref{memLimit}.
\end{proof}

\begin{remark}
The theorem state that one cannot determine whether or not $P_{\mathcal{Z}}=\star$ from the final network. However, if we used a variant of corollary \ref{memLimit} we could get a result showing that the final network will not compute $p_s$ accurately. More precisely, after training the network on $2^{n/24-1}$ pairs $(x,p_s(x))$, the network will not be able to compute $p_s(x)$ with accuracy $1/2+\omega(2^{-n/48})$ with a probability of $\omega(2^{-n/24})$. 
\end{remark}

We can also use this reasoning to prove theorem \ref{thm1'}, which is restated below.

\begin{theorem}
Let $\epsilon>0$, and $P_{\F}$ be a probability distribution over functions with a cross-predictability of $\mathrm{c_p}=o(1)$. For each $n > 0$, let $(f,g)$ be a neural net of polynomial size in $n$ such that each edge weight is recorded using $O(\log(n))$ bits of memory. Run stochastic gradient descent on $(f,g)$ with at most $\mathrm{c_p}^{-1/24}$ time steps and with $o(|\log(\mathrm{c_p})|/\log(n))$ edge weights updated per time step.  For all sufficiently large $n$, this algorithm fails at learning functions drawn from $P_{\F}$ with accuracy $1/2 + \epsilon$.
\end{theorem}

\begin{proof}
Consider a data structure that consists of a neural net $(f,g')$ and a boolean value $b$. Now, consider training $(f,g)$ with any such coordinate descent algorithm while using the data structure to store the current value of the net. Also, in each time step, set $b$ to $True$ if the net computed the output corresponding to the sampled input correctly and $False$ otherwise. This constitutes a data structure with a polynomial amount of memory that is divided into variables that are $O(\log n)$ bits long, such that $o(|\log(\mathrm{c_p})|/\log(n))$ variables change value in each time step. As such, by corollary \ref{memLimit}, one cannot determine whether the samples are actually generated by a random parity function or whether they are simply random elements of $\B^{n+1}$ from the data structure's final value with accuracy $1/2+\omega(\mathrm{c_p}^{1/24})$. In particular, one cannot determine which case holds from the final value of $b$. If the samples were generated randomly, it would compute the final output correctly with probability $1/2$, so $b$ would be equally likely to be $True$ or $False$. So, when it is trained on a random parity function, the probability that $b$ ends up being $True$ must be at most $1/2+O(\mathrm{c_p}^{1/24})$. Therefore, it must compute the final output correctly with probability $1/2+O(\mathrm{c_p}^{1/24})$.
\end{proof}

\subsection{Proof of Theorem \ref{thm2}}
Before we talk about the effectiveness of noisy gradient descent, we need to formally define it. So, we define the following. 


\vspace{1 cm}
\noindent
{\em NoisyGradientDescentStep(f, G, $P_\mathcal{Z}$, L, $\gamma$, $\delta$)}:
\begin{enumerate}
\item For each $(v,v')\in E(G)$, set 
\begin{align}
    w'_{v,v'}=w_{v,v'}-\gamma \frac{\partial E_{(X,Y)\sim P_\mathcal{Z}} [L(eval_{(f,G)}(X)-Y)]}{\partial w_{v,v'}}+\delta
\end{align}

\item Return the graph that is identical to $G$ except that its edge weights are given by $w'$.
\end{enumerate}


\vspace{1 cm}
\noindent
{\em NoisyGradientDescentAlgorithm(f, G, $P_\mathcal{Z}$, L, $\gamma$, $\Delta$, t)}:
\begin{enumerate}
\item Set $G_0=G$.

\item For each $0\le i<t$:
\begin{enumerate}

\item Generate $\delta^{(i)}$ by independently drawing $\delta^{(i)}_{v,v'}$ from $\Delta$ for each $(v,v')\in E(G)$.

\item Set $G_{i+1}=NoisyGradientDescentStep(f, G_i, h, P_\mathcal{Z}, L, \gamma, \delta^{(i)})$.

\end{enumerate}

\item Return $G_t$.
\end{enumerate}

\vspace{1 cm}
The case where one of the derivatives is very large causes enough problems that it is convenient to define versions of the algorithms that treat the derivative of the loss function with respect to a given edge weight on a given input as having some maximum value if it is actually larger. Then we can prove results for these algorithms, and argue that if none of the derivatives are that large using the normal versions must yield the same result. As such, we define the following:

\begin{definition}
For every $B>0$ and $x\in\mathbb{R}$, let $\Psi_B(x)$ be $B$ if $x>B$, $-B$ if $x<-B$, and $x$ otherwise.
\end{definition}

\vspace{1 cm}
\noindent
{\em BoundedNoisyGradientDescentStep(f, G, $P_\mathcal{Z}$, L, $\gamma$, $\delta$, $B$)}:
\begin{enumerate}

\item For each $(v,v')\in E(G)$, set 
        \begin{align}
            w'_{v,v'}=w_{v,v'}-\gamma E_{(X,Y)\sim P_\mathcal{Z}} \left[ \Psi_B\left(\frac{\partial L(eval_{(f,G)}(X)-Y)}{\partial w_{v,v'}}\right)\right]+\delta
        \end{align}

\item Return the graph that is identical to $G$ except that its edge weights are given by $w'$.
\end{enumerate}

\vspace{1 cm}
\noindent
{\em BoundedNoisyGradientDescentAlgorithm(f, G, $P_\mathcal{Z}$, L, $\gamma$, $\Delta$, $B$, t)}:
\begin{enumerate}
\item Set $G_0=G$.

\item For each $0\le i<t$:
\begin{enumerate}

\item Generate $\delta^{(i)}$ by independently drawing $\delta^{(i)}_{v,v'}$ from $\Delta$ for each $(v,v')\in E(G)$.

\item Set $G_{i+1}=BoundedNoisyGradientDescentStep(f, G_i, h, P_\mathcal{Z}, L, \gamma, \delta^{(i)}, B)$.

\end{enumerate}

\item Return $G_t$.
\end{enumerate}

In order to prove that noisy gradient descent fails to learn a random parity function, we first show that the gradient will be very small, and then argue that the noise will drown it out. The first steps are the following basic inequalities.

\begin{lemma}\label{new-pred}
Let $n>0$ and $f:\B^{n+1}\rightarrow \mathbb{R}$. Also, let $X$ be a random element of $\B^n$ and $Y$ be a random element of $\B$ independent of $X$. Then
\[\sum_{s\subseteq[n]} (\E f(X,Y)-\E f(X,p_s(X)) )^2\le \E f^2(X,Y) \]
\end{lemma}

\begin{proof}
For each $x\in\B^n$, let $g(x)=f(x,1)-f(x,0)$.

\begin{align}
&\sum_{s\subseteq[n]} (\E[f(X,Y)]-\E[f(X,p_s(X))])^2\\
&=\sum_{s\subseteq[n]} \left(2^{-n-1}\sum_{x\in\B^n} (f(x,0)+f(x,1)-2f(x,p_s(x)))\right)^2\\
&=\sum_{s\subseteq[n]} \left(2^{-n-1}\sum_{x\in\B^n} g(x)(-1)^{p_s(x)}\right)^2   \label{pars1} \\
&=2^{-2n-2}\sum_{x_1,x_2\in\B^n,s\subseteq[n]}g(x_1)(-1)^{p_s(x_1)}\cdot g(x_2)(-1)^{p_s(x_2)}\\
&=2^{-2n-2}\sum_{x_1,x_2\in\B^n}g(x_1)g(x_2) \sum_{s\subseteq[n]}(-1)^{p_s(x_1)}(-1)^{p_s(x_2)}\\
&=2^{-2n-2}\sum_{x\in\B^n}2^n g^2(x) \label{pars2} \\
&= 2^{-n-2}\sum_{x\in\B^n} [f(x,1)-f(x,0)]^2  \\
&\le 2^{-n-1}\sum_{x\in\B^n} f^2(x,1)+f^2(x,0)\\
&=\E[f^2(X,Y)]
\end{align}
where we note that the equality from \eqref{pars1} to \eqref{pars2} is Parserval's identity for the Fourier-Walsh basis (here we used Boolean outputs for the parity functions). \end{proof}
Note that by the triangular inequality the above implies
\begin{align}
\Var_F \E_{X}f(X,F(X))\le 2^{-n} \E_{X,Y}f^2(X,Y).
\end{align}
As mentioned earlier, this is similar to Theorem 1 in \cite{ohad} that requires in addition the function to be the gradient of a 1-Lipschitz loss function.   



We also mention the following corollary of Lemma \ref{new-pred} that results from Cauchy-Schwarz.

\begin{corollary}\label{parityAverage}
Let $n>0$ and $f:\B^{n+1}\rightarrow \mathbb{R}$. Also, let $X$ be a random element of $\B^n$ and $Y$ be a random element of $\B$ independent of $X$. Then
\[\sum_{s\subseteq[n]} |E[f((X,Y))]-E[f((X,p_s(X)))]|\le 2^{n/2}\sqrt{E[f^2((X,Y))]}.\]
\end{corollary}
In other words, the expected value of any function on an input generated by a random parity function is approximately the same as the expected value of the function on a true random input. This in turn implies the following:

\begin{lemma}
Let $(f, g)$ be a neural net with $m$ edges, $B, \gamma,\sigma>0$, and $L:\mathbb{R}\rightarrow\mathbb{R}$ be a differentiable function. Also, let $\star$ be the uniform distribution on $\B^{n+1}$, and for each $s\subseteq [n]$, let $\rho_s$ be the probability distribution of $(X,p_s(X))$ when $X$ is chosen randomly from $\B^n$. Next, let $Q_\star$ be the probability distribution of the output of BoundedNoisyGradientDescentStep(f, g, $\star$, L, $\gamma$, $\delta$, B) when $\delta\sim\mathcal{N}(0,\sigma^2 I)$ and $Q_s$ be the probability distribution of the output of BoundedNoisyGradientDescentStep(f, g, $\rho_s$, L, $\gamma$, $\delta$, B) when $\delta\sim\mathcal{N}(0,\sigma^2 I)$ for each $s\subseteq[n]$. Then
\[\sum_{s\subseteq[n]} ||Q_\star-Q_s||_1\le \gamma \sqrt{\frac{1}{\pi\sigma^2}\sum_{(x,y)\in \B^{n+1},(v,v')\in E(g)} \left(\frac{\partial [L(eval_{(f,g)}(x)-y)]}{\partial w_{v,v'}}\right)^2}.\]

\[\sum_{s\subseteq[n]} ||Q_\star^{(T)}-Q^{(T)}_s||_1\le \gamma B \sqrt{\frac{m 2^{n+1}}{\pi\sigma^2}}\]


\end{lemma}

\begin{proof}
Let $w^{(\star)}\in\mathbb{R}^m$ be the edge weights of the graph output by BoundedNoisyGradientDescentStep $(f, g, \star, L, \gamma, 0, B)$ and $w^{(s)}\in\mathbb{R}^m$ be the edge weights of the graph output by BoundedNoisyGradientDescentStep $(f, g, \rho_s, L, \gamma, 0, B)$ for each $s$. Observe that
\begin{align*}
&\sum_{s\subseteq[n]} ||w^{(s)}-w^{(\star)}||_2\\
&\le 2^{n/2}\sqrt{\sum_{s\subseteq[n]} ||w^{(s)}-w^{(\star)}||_2^2}\\
&\le 2^{n/2}\sqrt{\sum_{(v,v')\in E(g)} \sum_{s\subseteq[n]} \left(w_{(v,v')}^{(s)}-w_{(v,v')}^{(\star)}\right)^2}\\
&\le 2^{n/2} \sqrt{\sum_{(v,v')\in E(g)} 2^{-n-1}\sum_{(x,y)\in \B^{n+1}} \gamma^2 \Psi_B^2\left( \frac{\partial [L(eval_{(f,g)}(x)-y)]}{\partial w_{v,v'}}\right)}\\
&\le 2^{n/2} \sqrt{\sum_{(v,v')\in E(g)} 2^{-n-1}\sum_{(x,y)\in \B^{n+1}} \gamma^2 B^2}\\
&= \gamma B \sqrt{m 2^n}
\end{align*}
where the second to last inequality follows by applying the previous lemma to the formula in BoundedNoisyGradientDescentStep for the changes in edge weight. Next, observe that $Q_\star$ is a Gaussian distribution with mean $w^{(\star)}$ and covariance $\sigma^2 I$. Similarly, $Q_s$ is a Gaussian distribution with mean $w^{(s)}$ and covariance $\sigma^2 I$ for each $s$. As such,

\begin{align*}
&||Q_\star-Q_s||_1\\
&=2-2\int_{-\infty}^{\infty} \frac{1}{\sqrt{2\pi\sigma^2}}\min(e^{-x^2/2\sigma^2},e^{-(x-||w^{(\star)}-w^{(s)}||_2)^2/2\sigma^2}) dx\\
&=2-2\int_{-\infty}^{||w^{(\star)}-w^{(s)}||_2/2} \frac{1}{\sqrt{2\pi\sigma^2}} e^{-(x-||w^{(\star)}-w^{(s)}||_2)^2/2\sigma^2} dx-2\int_{||w^{(\star)}-w^{(s)}||_2/2}^{\infty} \frac{1}{\sqrt{2\pi\sigma^2}} e^{-x^2/2\sigma^2} dx\\
&=2-2\int_{-\infty}^{-||w^{(\star)}-w^{(s)}||_2/2} \frac{1}{\sqrt{2\pi\sigma^2}} e^{-x^2/2\sigma^2} dx-2\int_{||w^{(\star)}-w^{(s)}||_2/2}^{\infty} \frac{1}{\sqrt{2\pi\sigma^2}} e^{-x^2/2\sigma^2} dx\\
&=2\int_{-||w^{(\star)}-w^{(s)}||_2/2}^{||w^{(\star)}-w^{(s)}||_2/2} \frac{1}{\sqrt{2\pi\sigma^2}} e^{-x^2/2\sigma^2} dx\\
&\le 2||w^{(\star)}-w^{(s)}||_2/\sqrt{2\pi\sigma^2}
\end{align*}
for all $s$. 

Therefore,
\begin{align*}
&\sum_{s\subseteq[n]} ||Q_\star-Q_s||_1\\
&\le \sum_{s\subseteq[n]} 2||w^{(s)}-w^{(\star)}||_2/\sqrt{2\pi\sigma^2}\\
&\le \gamma B \sqrt{m 2^{n+1}/\pi\sigma^2}
\end{align*}
\end{proof}

This allows us to prove the following.

\begin{corollary}
Let $(f, g)$ be a neural net with $m$ edges, $B, \gamma,\sigma, T>0$, and $L:\mathbb{R}\rightarrow\mathbb{R}$ be a differentiable function. Also, let $\star$ be the uniform distribution on $\B^{n+1}$, and for each $s\subseteq [n]$, let $\rho_s$ be the probability distribution of $(X,p_s(X))$ when $X$ is chosen randomly from $\B^n$. Next, for each $0\le t\le T$, let $Q_\star^{(t)}$ be the probability distribution of the output of BoundedNoisyGradientDescentAlgorithm(f, g, $\star$, L, $\gamma$, $\mathcal{N}(0,\sigma^2 I)$, $B$, t) and $Q_s^{(t)}$ be the probability distribution of the output of BoundedNoisyGradientDescentAlgorithm(f, g, $\rho_s$, L, $\gamma$,$\mathcal{N}(0,\sigma^2 I)$,$B$, t) for each $s\subseteq[n]$. Then
\[\sum_{s\subseteq[n]} ||Q_\star^{(T)}-Q^{(T)}_s||_1\le \gamma B T \sqrt{\frac{m 2^{n+1}}{\pi\sigma^2}}\]
\end{corollary}

\begin{proof}
Clearly, $Q_s^{(0)}=Q_\star^{(0)}$ for all $s$. Now, for each $s\subseteq [n]$ and $t>0$, let $Q_{\star s}^{(t)}$ be the probability distribution of the output of BoundedNoisyGradientDescentStep(f, $G'_t$, $\rho_s$, L, $\gamma$, $\mathcal{N}(0,\sigma^2 I)$, $B$), where $G'_t\sim Q_\star^{(t-1)}$. Next, observe that $Q_\star^{(t)}$ is the probability distribution of the output of BoundedNoisyGradientDescentStep(f, $G'_t$, $\star$, L, $\gamma$, $\mathcal{N}(0,\sigma^2 I)$, $B$). So, by the previous lemma, 
\[\sum_{s\subseteq[n]} ||Q_{\star}^{(t)}-Q_{\star s}^{(t)}||_1\le \gamma B \sqrt{\frac{m 2^{n+1}}{\pi\sigma^2}}\]

Also, for each $s$ and $t$, $Q_s^{(t)}$ is the probability distribution of the output of BoundedNoisyGradientDescentStep(f, $G''_t$, $\rho_s$, L, $\gamma$, $\mathcal{N}(0,\sigma^2 I)$, $B$), where $G''_t\sim Q_s^{(t-1)}$. So, $||Q_{\star s}^{(t)}-Q_s^{(t)}||_1\le ||Q_{\star}^{(t-1)}-Q_s^{(t-1)}||_1$. Thus,
\[\sum_{s\subseteq[n]} ||Q_{\star}^{(t)}-Q_{s}^{(t)}||_1\le \gamma B \sqrt{\frac{m 2^{n+1}}{\pi\sigma^2}}+\sum_{s\subseteq[n]} ||Q_{\star}^{(t-1)}-Q_s^{(t-1)}||_1\]
by the triangle inequality. The desired result follows by induction.
\end{proof}

This in turn implies the following elaboration of Theorem \ref{thm2}.

\begin{theorem}
Let $(f, g)$ be a neural net with $m$ edges, $B, \gamma,\sigma, T>0$, and $L:\mathbb{R}\rightarrow\mathbb{R}$ be a differentiable function. Also, let $\star$ be the uniform distribution on $\B^{n+1}$, and for each $s\subseteq [n]$, let $\rho_s$ be the probability distribution of $(X,p_s(X))$ when $X$ is chosen randomly from $\B^n$. Then for a random $S\subseteq[n]$ and $X\in\B^n$ the probability that the net output by BoundedNoisyGradientDescentAlgorithm(f, g, $\rho_S$, L, $\gamma$,$\mathcal{N}(0,\sigma^2 I)$, B,T) computes $p_S(X)$ correctly is at most 
\[1/2+\gamma BT\sqrt{\frac{m 2^{-n}}{2\pi\sigma^2}}\]
\end{theorem}

\begin{proof}
Let $Q_\star$ be the probability distribution of the output of BoundedNoisyGradientDescentAlgorithm(f, g, $\star$, L, $\gamma$, $\mathcal{N}(0,\sigma^2 I)$, B, T) and $Q_s$ be the probability distribution of the output of BoundedNoisyGradientDescentAlgorithm(f, g, $\rho_s$, L, $\gamma$,$\mathcal{N}(0,\sigma^2 I)$, B, T) for each $s\subseteq[n]$. By the previous corollary, we know that
\[\sum_{s\subseteq[n]} ||Q_\star-Q_s||_1\le \gamma B T \sqrt{\frac{m 2^{n+1}}{\pi\sigma^2}}\]

Now, let $G^\star\sim Q_\star$ and $G^s\sim Q_s$ for each $s$. Also, for each $(x,y)\in\B^{n+1}$, let $R_{(x,y)}$ be the set of all neural nets that output $y$ when $x$ is input to them. For any $s\subseteq[n]$ and $x\in\B^n$, we have that
\begin{align*}
P[eval_{(f,G^s)}(x)=p_s(x)]&= P[G^s\in R_{(x,p_s(x)}]\\
&\le P[G^\star\in R_{(x,p_s(x))}]+||Q_\star-Q_s||_1/2
\end{align*}
That means that
\begin{align*}
&\sum_{s\subseteq[n],x\in\B^n} P[eval_{(f,G^s)}(x)=p_s(x)]\\
&\le \sum_{s\subseteq[n],x\in\B^n} P[G^\star\in R_{(x,\rho_s(x))}]+||Q_\star-Q_s||_1/2\\
&= \sum_{x\in\B^n} 2^{n-1}+2^n\sum_{s\subseteq[n]} ||Q_\star-Q_s||_1/2\\
&= 2^{2n-1}+2^n \gamma B T \sqrt{\frac{m 2^{n-1}}{\pi\sigma^2}}
\end{align*}
Dividing both sides by $2^{2n}$ yields the desired conclusion.
\end{proof}

\subsection{Proof of Theorem \ref{thm2'}}
We prove here an elaboration of Theorem \ref{thm2'}.
\begin{theorem}
Let $P_{\X}$ with $\X=\mathcal{D}^n$ for some set $\mathcal{D}$ and $P_{\F}$ such that the output distribution is balanced,\footnote{Non-balanced cases can be handled by modifying definitions appropriately.} i.e., $\pp\{F(X)=0\}=\pp\{F(X)=1\}+o_n(1)$ when $(X,F) \sim P_{\X^n} \times P_{\F}$. Let $\Pi:=\mathrm{Pred}( P_{\X}, P_{\F} )$ and $F \sim P_{\F}$.

For each $n > 0$, take a neural net of size $|E|$, with any differentiable non-linearity and any initialization of the weights $W^{(0)}$, and train it with gradient descent with learning rate $\gamma$, any differentiable loss function, gradients computed on the population distribution $P_\X$ with labels from $F$, an overflow range of $A$, additive Gaussian noise of variance $\sigma^2$, and $S$ steps, i.e.,        
\begin{align}
W^{(t)}= W^{(t-1)} - \gamma \left[\nabla \E_{X \sim P_\X} L(W^{(t-1)}(X),F(X)) \right]_A + Z^{(t)}, \quad t=1,\dots, S,
\end{align}
where $\{Z^{(t)}\}_{t \in [S]}$ are i.i.d.\ $\mathcal{N}(0,\sigma^2)$. 
Then,
\begin{align}
\pp\{ W^{(S)}(X)=F(X)\}&=1/2+\mathrm{grapes},\\
\mathrm{grapes}&=\gamma \frac{1}{\sigma} A \Pi^{1/4} |E|^{1/2} S. 
\end{align}
\end{theorem}

Note that this theorem uses the BoundedNoisyGradientDescentAlgorithm, we simply re-wrote it to make the statement self-contained, using $W^{(t)}(X)$ as the evaluation of the neural net with weights $W^{(t)}$ on an input $X$. 

\begin{proof}



For $t=1,\dots, S$ and $H \in \{F,\star\}$, define 
\begin{align}
W_H^{(t)}= W_H^{(t-1)} - \gamma \left[\nabla \E_{(X,Y) \sim \rho_H} L(W_H^{(t-1)}(X),Y) \right]_A + Z^{(t)},       
\end{align}
where 
\begin{align}
\rho_H(x,y)=
\begin{cases}
P_\X(x) U_\Y(y) & \text{ if } H= \star,\\
P_\X(x) \delta_{F(X)}(y) & \text{ if } H= F,
\end{cases}
\end{align}
and let $Q_H^{(t)}$ be the distribution of $W_H^{(t)}$. 


{\it One-step bound.} Denote $Q_H:=Q_H^{(1)}$ and $w_H:=\gamma \left[\nabla \E_{(X,Y) \sim \rho_H} L(W^{0}(X),Y) \right]_A$, $H \in \{F,\star\}$. We have
\begin{align}
 d( Q_{F} , Q_{\star} )_{TV} = 1 - 2 P_e(\rho_F,\rho_\star)  
\end{align}
where $P_e(\rho_F,\rho_\star)$ is the probability of error of the optimal (MAP) test for the hypothesis test between $\rho_F$ and $\star$ with equiprobable priors, i.e., 
\begin{align}
P_e(\rho_F,\rho_\star) =  \pp \{ N_\sigma  \ge \| w_F - w_\star \|_2/2 \}
\end{align}
and therefore 
\begin{align}
 d( Q_{F} , Q_{\star} )_{TV} &= \pp \{ N_\sigma  \in  [-\| w_F - w_\star \|_2/2, \| w_F - w_\star \|_2/2] \}\\
 &\le \frac{1}{\sqrt{2\pi\sigma^2}} \| w_F - w_\star \|_2.
\end{align}
Using Cauchy-Schwarz and previous inequality, we have
\begin{align}
 \E_F  d( Q_{F} , Q_{\star} )_{TV}
& \le (\E_F  d( Q_{F} , Q_{\star} )_{TV}^2 )^{1/2}\\
& \le  \frac{1}{\sqrt{2\pi\sigma^2}} (\E_F \| w_F - w_\star \|_2^2   )^{1/2}.\label{tv2}
\end{align}
Let us focus now on the term
\begin{align}
\E_F \| w_F - w_\star \|_2^2  &=
\sum_{(u,v) \in  E(g)} \E_F  (w_F(u,v) - w_\star(u,v) )^2.\label{f1}
\end{align}
Note that 
\begin{align}
 \E_F  (w_F(u,v) - w_\star(u,v) )^2 &= 
  \E_F  (\E_{(X,Y) \sim \rho_F} \psi_{u,v}(X,Y) -\E_{(X,Y) \sim \rho_\star} \psi_{u,v}(X,Y)  )^2 
\end{align}
where 
\begin{align}
\psi_{u,v}(x,y)&:=\gamma \left[ \frac{\partial  L(eval_{(\phi,g)}(x)-y)}{\partial w_{u,v}} \right]_A.
\end{align}
We have  
\begin{align}
& \E_F  (\E_{(X,Y) \sim \rho_F} \psi_{u,v}(X,Y) -\E_{(X,Y) \sim \rho_\star} \psi_{u,v}(X,Y)  )^2\\
 &=  \E_F  \left( \E_{(X,Y) \sim \rho_\star} \psi_{u,v}(X,Y)\left(1- \frac{P_{\X}(X) \delta_{F(X)}(Y)}{P_{\X}(X) U_\Y(Y)}\right)  \right)^2 \\
 &=  \E_F \left( \E_{(X,Y) \sim \rho_\star} \psi_{u,v}(X,Y) \left(1- 2\delta_{F(X)}(Y)\right)  \right)^2\\ 
 &=  \E_F  \left( \E_Z g(Z) h_F(Z)  \right)^2
\end{align}
where $Z:=(X,Y) \sim P_\Z := P_\X \times U_\Y$, $g(Z):=\psi_{u,v}(X,Y)$, and $h_F(Z)=1$ if $F(X)\ne Y$ and $h_F(Z)=-1$ otherwise. Therefore, 
\begin{align}
  \E_F  \left( \E_Z g(Z) h_F(Z)  \right)^2 &= \E_F  \langle g  ,h_F   \rangle_{P_\Z}^2  \\
&= \E_F  \langle g^{\otimes 2} ,h_F^{\otimes 2}  \rangle_{P_\Z}  \\
&=  \langle g^{\otimes 2} ,\E_F  h_F^{\otimes 2}  \rangle_{P_\Z}  \\ 
&\le  \| g^{\otimes 2} \|_{P_\Z}  \| \E_F  h_F^{\otimes 2}  \|_{P_\Z}  \\ 
&= \| g  \|_{P_\Z}^2  \langle \E_F  h_F^{\otimes 2}, \E_{F'}  h_{F'}^{\otimes 2} \rangle_{P_\Z}^{1/2}  \\ 
&=  \| g  \|_{P_\Z}^2   (\E_{F,F'} \langle   h_F^{\otimes 2},  h_{F'}^{\otimes 2}  \rangle_{P_\Z})^{1/2}  \\
&=  \| g  \|_{P_\Z}^2   (\E_{F,F'} \langle   h_F ,  h_{F'}   \rangle_{P_\Z}^2)^{1/2}  \\
&=  \| g  \|_{P_\Z}^2   (\mathrm{Pred}( P_{\X}, P_{\F} ) )^{1/2}.
\end{align}
Putting the pieces together, we obtain  
\begin{align}
\E_F \| w_F - w_\star \|_2^2  &\le   \mathrm{Pred}( P_{\X}, P_{\F} )^{1/2} \| \E_Z \psi(Z) \|_2^2.
\end{align}
and
\begin{align}
\E_F  d( Q_{F} , Q_{\star} )_{TV} &\le \frac{1}{\sigma} \mathrm{Pred}( P_{\X}, P_{\F} )^{1/4} \| \E_{X,Y}  [\nabla L(W^{(0)}(X),Y)]_A  \|_2
\end{align}

{\it Multi-step bound.} We now proceed with a cumulative argument.
For $t \in [S+1]$ $H,h \in \{F,\star\}$, define 
\begin{align}
W_{H,h}^{(t-1)}= W_H^{(t-1)} - \gamma \left[\nabla \E_{(X,Y) \sim \rho_h} L(W_H^{(t-1)}(X),Y) \right]_A + Z^{(t)},       
\end{align}
and denote by $Q^{(t-1)}_{H,h}$ the distribution of $W_{H,h}^{(t-1)}$.

Using the triangular and Data-Processing inequalities, we have 
\begin{align}
 \E_F d( Q^{(t)}_{F} , Q^{(t)}_{\star} )_{TV}
& \le \E_F d( Q^{(t-1)}_{F,F},Q^{(t-1)}_{\star, F})_{TV} + \E_F d(Q^{(t-1)}_{\star, F}, Q^{(t-1)}_{\star,\star} )_{TV}\\
& \le \E_F d( Q^{(t-1)}_{F,F},Q^{(t-1)}_{\star, F})_{TV} +  \frac{1}{\sigma} \Pi^{1/4} \| \E_{X,Y} \gamma_{t}  [\nabla L(W_{\star}^{(t-1)}(X),Y)]_A  \|_2  \\
& \le  \E_F d( Q^{(t-1)}_{F},Q^{(t-1)}_{\star})_{TV} +  \frac{1}{\sigma} \Pi^{1/4} \| \E_{X,Y} \gamma_{t} [\nabla L(W_{\star}^{(t-1)}(X),Y)]_A  \|_2.
\end{align}

We thus obtain 
\begin{align}
\E_F  d( Q^{(S)}_{F} , Q^{(S)}_{\star} )_{TV}& \le  \frac{\Pi^{1/4}}{\sigma}  \sum_{t=1}^S  \gamma_{t} \| \E_{X,Y}  [\nabla L(W_{\star}^{(t-1)}(X),Y)]_A  \|_2.
\end{align}

{\it Indistinguishability.} Finally, by the definition of the total variation distance, 
\begin{align}
\pp\{W_F^{(S)}(X) = F(X)  \} &\le 
\pp\{ W_\star^{(S)}(X)= F(X)  \} + \E_F d( Q^{(S)}_{F} , Q^{(S)}_{\star} )_{TV}\\
&\le 1/2 +  \frac{\Pi^{1/4}}{\sigma}  \sum_{t=1}^S  \gamma_{t} \| \E_{X,Y}  [\nabla L(W_{\star}^{(t-1)}(X),Y)]_A  \|_2,
\end{align}
and
\begin{align}
 \| \E_{X,Y}  [\nabla L(W_{\star}^{(t-1)}(X),Y)]_A  \|_2 \le A \sqrt{E}.
\end{align}

\end{proof}

\subsection{Proof of Theorem \ref{thm3}}

Our next goal is to make a similar argument for stochastic gradient descent. We argue that if we use noisy SGD to train a neural net on a random parity function, the probability distribution of the resulting net is similar to the probability distribution of the net we would get if we trained it on random values in $\B^{n+1}$. This will be significantly harder to prove than in the case of noisy gradient descent, because while the difference in the expected gradients is exponentially small, the gradient at a given sample may not be. As such, drowning out the signal will require much more noise. However, before we get into the details, we will need to formally define a noisy version of SGD, which is as follows.

\vspace{1 cm}
\noindent
{\em NoisySampleGradientDescentStep(f, G, $Y$, $X$, L, $\gamma$, B, $\delta$)}:
\begin{enumerate}
\item For each $(v,v')\in E(G)$: \begin{enumerate}
\item Set 
\[w'_{v,v'}=w_{v,v'}-\gamma \frac{\partial L(eval_{(f,G)}(X)-Y)}{\partial w_{v,v'}}+\delta_{v,v'}\]

\item If $w'_{v,v'}<-B$, set $w'_{v,v'}=-B$.

\item If $w'_{v,v'}>B$, set $w'_{v,v'}=B$.

\end{enumerate}

\item Return the graph that is identical to $G$ except that its edge weight are given by the $w'$.
\end{enumerate}

\vspace{1 cm}
\noindent
{\em NoisyStochasticGradientDescentAlgorithm(f, G, $P_\mathcal{Z}$, L, $\gamma$, B, $\Delta$, t)}:
\begin{enumerate}
\item Set $G_0=G$.

\item If any of the edge weights in $G_0$ are less than $-B$, set all such weights to $-B$.

\item If any of the edge weights in $G_0$ are greater than $B$, set all such weights to $B$.

\item For each $0\le i<t$:
\begin{enumerate}

\item Draw $(X_i, Y_i)\sim P_\mathcal{Z}$, independently of all previous values.

\item Generate $\delta^{(i)}$ by independently drawing $\delta^{(i)}_{v,v'}$ from $\Delta$ for each $(v,v')\in E(G)$.

\item Set $G_{i+1}=NoisySampleGradientDescentStep(f, G_i, Y_i, X_i, L, \gamma, B, \delta)$
\end{enumerate}

\item Return $G_t$.
\end{enumerate}

\subsubsection{Uniform noise and SLAs}

The simplest way to add noise in order to impede learning a parity function would be to add noise drawn from a uniform distribution in order to drown out the information provided by the changes in edge weights. More precisely, consider setting $\Delta$ equal to the uniform distribution on $[-C,C]$. If the change in each edge weight prior to including the noise always has an absolute value less than $D$ for some $D<C$, then with probability $\frac{C-D}{C}$, the change in a given edge weight including noise will be in $[-(C-D),C-D]$. Furthermore, any value in this range is equally likely to occur regardless of what the change in weight was prior to the noise term, which means that the edge's new weight provides no information on the sample used in that step. If $D/C=o(nE(G)/\ln(n))$ then this will result in there being $o(n/\log(n))$ changes in weight that provide any relevant information in each timestep. So, the resulting algorithm will not be able to learn the parity function by an extension of corollary $\ref{memLimit}$. This leads to the following result:

\begin{theorem}
Let $n>0$, $\gamma>0$, $D>0$, $t=2^{o(n)}$, $(f,G)$ be a normal neural net of size polynomial in $n$, and $L:\mathbb{R}\rightarrow\mathbb{R}$ be a smooth, convex, symmetric function with $L(0)=0$. Also, let $\Delta$ be the uniform probability distribution on $[-D|E(G)|,D|E(G)|]$. Now, let $S$ be a random subset of $[n]$ and $P_\mathcal{Z}$ be the probability distribution $(X,p_S(X))$ when $X$ is drawn randomly from $\B^n$. Then when NoisyStochasticGradientDescentAlgorithm(f, G, $P_\mathcal{Z}$, L, $\gamma$, $\infty$, $\Delta$, t) is run on a computer that uses $O(\log(n))$ bits to store each edge's weight, with probability $1-o(1)$ either there is at least one step when the adjustment to one of the weights prior to the noise term has absolute value greater than $D$ or the resulting neural net fails to compute $p_S$ with nontrivial accuracy.
\end{theorem}

This is a side result and we provide a concise proof. 
\begin{proof}

Consider the following attempt to simulate NoisyStochasticGradientDescentAlgorithm(f, G, $P_\mathcal{Z}$, L, $\gamma$, $\infty$, $\Delta$, t) with a sequential learning algorithm. First, independently draw $b^{t'}_{v,v'}$ from the uniform probability distribution on $[-D|E(G)|+D,D|E(G)|-D]$ for each $(v,v')\in E(G)$ and $t'\le t$. Next, simulate NoisyStochasticGradientDescentAlgorithm(f, G, $P_\mathcal{Z}$, L, $\gamma$, $\infty$, $\Delta$, t) with the following modifications. If there is ever a step where one of the adjustments to the weights before the noise term is added in is greater than $D$, record ``failure" and give up. If there is ever a step where more than $n/\ln^2(n)$ of the weights change by more than $D|E(G)|-D$ after including the noise record "failure" and give up. Otherwise, record a list of which weights changed by more than $D|E(G)|-D$ and exactly what they changed by. In all subsequent steps, assume that $W_{v,v'}$ increased by $b^{t'}_{v,v'}$ in step $t'$ unless the amount it changed by in that step is recorded.

First, note that if the values of $b$ are computed in advance, the rest of this algorithm is a sequential learning algorithm that records $O(n/\log(n))$ bits of information per step and runs for a subexponential number of steps. As such, any attempt to compute $p_S(X)$ based on the information provided by its records will have accuracy $1/2+o(1)$ with probability $1-o(1)$. Next, observe that in  a given step in which all of the adjustments to weights before the noise is added in are at most $D$, each weight has a probability of changing by more than $D|E(G)|-D$ of at most $1/|E(G)|$ and these probabilities are independent. As such, with probability $1-o(1)$, the algorithm will not record "failure" as a result of more than $n/\ln^2(n)$ of the weights changing by more than $D|E(G)|-D$. Furthermore, the probability distribution of the change in the weight of a given vertex conditioned on the assumption that said change is at most $D|E(G)|-D$ and a fixed value of said change prior to the inclusion of the noise term that has an absolute value of at most $D$ is the uniform probability distribution on $[-D|E(G)|+D,D|E(G)|-D]$. As such, substituting the values of $b^{t'}_{v,v'}$ for the actual changes in weights that change by less than $D|E(G)|-D$ has no effect on the probability distribution of the resulting graph. As such, the probability distribution of the network resulting from NoisyStochasticGradientDescentAlgorithm(f, G, $P_\mathcal{Z}$, L, $\gamma$, $\infty$, $\Delta$, t) if none of the weights change by more than $D$ before noise is factored in differs from the probabiliy distribution of the network generated by this algorithm if it suceeds by $o(1)$. Thus, the fact that the SLA cannot generate a network that computed $p_S$ with nontrivial accuracy implies that NoisyStochasticGradientDescentAlgorithm(f, G, $P_\mathcal{Z}$, L, $\gamma$, $\infty$, $\Delta$, t) also fails to generate a network that computes $p_S$ with nontrivial accuracy.
\end{proof}

\begin{remark}
At first glance, the amount of noise required by this theorem is ridiculously large, as it will almost always be the dominant contribution to the change in any weight in any given step. However, since the noise is random it will tend to largely cancel out over a longer period of time. As such, the result of this noisy version of stochastic gradient descent will tend to be similar to the result of regular stochastic gradient descent if the learning rate is small enough. In particular, this form of noisy gradient descent will be able to learn to compute most reasonable functions with nontrivial accuracy for most sets of starting weights, and it will be able to learn to compute some functions with nearly optimal accuracy. Admittedly, it still requires a learning rate that is smaller than anything people are likely to use in practice.
\end{remark}
We next move to handling lower levels of noise. 

\subsubsection{Gaussian noise, noise accumulation, and blurring}
While the previous result works, it requires more noise than we would really like. The biggest problem with it is that it ultimately argues that even given a complete list of the changes in all edge weights at each time step, there is no way to determine the parity function with nontrivial accuracy, and this requires a lot of noise. However, in order to prove that a neural net optimized by NSGD cannot learn to compute the parity function, it suffices to prove that one cannot determine the parity function from the edge weights at a single time step. Furthermore, in order to prove this, we can use the fact that noise accumulates over multiple time steps and argue that the amount of accumulated noise is large enough to drown out the information on the function provided by each input.

More formally, we plan to do the following. First of all, we will be running NSGD with a small amount of Gaussian noise added to each weight in each time step, and a larger amount of Gaussian noise added to the initial weights. Under these circumstances, the probability distribution of the edge weights resulting from running NSGD on truly random input for a given number of steps will be approximately equal to the convolution of a multivariable Gaussian distribution with something else. As such, it would be possible to construct an oracle approximating the edge weights such that the probability distribution of the edge weights given the oracle's output is essentially a multivariable Gaussian distribution. Next, we show that given any function on $\B^{n+1}$, the expected value of the function on an input generated by a random parity function is approximately equal to its expected value on a true random input. Then, we use that to show that given a slight perturbation of a Gaussian distribution for each $z\in\B^{n+1}$, the distribution resulting from averaging togetherthe perturbed distributions generated by a random parity function is approximately the same as the distribution resulting from averaging together all of the perturbed distributions. Finally, we conclude that the probability distribution of the edge weights after this time step is essentially the same when the input is generated by a random parity function is it is when the input is truly random.

Our first order of business is to establish that the probability distribution of the weights will be approximately equal to the convolution of a multivariable Gaussian distribution with something else, and to do that we will need the following definition.

\begin{definition}
For $\sigma,\epsilon\ge 0$ and a probability distribution $\widehat{P}$, a probability distribution $P$ over $\mathbb{R}^m$ is a $(\sigma,\epsilon)$-blurring of $\widehat{P}$ if
\[||P-\widehat{P} * \mathcal{N}(0,\sigma I)||_1\le 2\epsilon\]
In this situation we also say that $P$ is a $(\sigma,\epsilon)$-blurring. If $\sigma\le 0$ we consider every probability distribution as being a $(\sigma,\epsilon)$-blurring for all $\epsilon$.
\end{definition}

The following are obvious consequences of this definition:
\begin{lemma}
Let $\mathcal{P}$ be a collection of $(\sigma,\epsilon)$-blurrings for some given $\sigma$ and $\epsilon$. Now, select $P\sim \mathcal{P}$ according to some probability distribution, and then randomly select $x\sim P$. The probability distribution of $x$ is also a $(\sigma,\epsilon)$-blurring.
\end{lemma}

\begin{lemma} \label{addBlur}
Let $P$ be a $(\sigma,\epsilon)$-blurring and $\sigma'>0$. Then $P*\mathcal{N}(0,\sigma' I)$ is a $(\sigma+\sigma',\epsilon)$-blurring
\end{lemma}

We want to prove that if the probability distribution of the weights at one time step is a blurring, then the probability distribution of the weights at the next time step is also a blurring. In order to do that, we need to prove that a slight distortion of a blurring is still a bluring. The first step towards that proof is the following lemma:

\begin{lemma}
Let $\sigma, B>0$, $m$ be a positive integer, $m\sqrt{2\sigma/\pi}<r\le1/(mB)$, and $f:\mathbb{R}^m\rightarrow\mathbb{R}^m$ such that $f(0)=0$, $|\frac{\partial f_i}{\partial x_j}(0)|=0$ for all $i$ and $j$, and $|\frac{\partial^2 f_i}{\partial x_j\partial x_{j'}}(x)|\le B$ for all $i$, $j$, $j'$, and all $x$ with $||x||_1<r$. Next, let $P$ be the probability distribution of $X+f(X)$ when $X\sim \mathcal{N}(0,\sigma I)$. Then $P$ is a $(\sigma,\epsilon)$-blurring for $\epsilon=\frac{4(m+2)m^2B\sqrt{2\sigma/\pi}+3m^5B^2\sigma}{8}+(1-Bmr)e^{-(r/2\sqrt{\sigma}-m/\sqrt{2\pi})^2/m}$.
\end{lemma}

\begin{proof}
First, note that for any $x$ with $||x||_1<r$ and any $i$ and $j$, it must be the case that $|\frac{\partial f_i}{\partial x_j}(x)|\le B||x||_1< Br$. That in turn means that for any $x,x'$ with $|x||_1,||x'||_1<r$ and any $i$, it must be the case that $|f(x)_i-f(x')_i|\le Br||x-x'||_1$ with equality only if $x=x'$. In particular, this means that for any such $x,x'$, it must be the case that $||f(x)-f(x')||_1\le mBr||x-x'||_1\le ||x-x'||_1$ with equality only if $x=x'$. Thus, $x+f(x)\ne x'+f(x')$ unless $x=x'$. Also, note that the bound on the second derivatives of $f$ implies that $|f_i(x)|\le B||x||_1^2/2$ for all $||x||_1<r$ and all $i$. This means that

\begin{align*}
&||P-\mathcal{N}(0,\sigma I)||_1\\
&\le 2-2\int_{x:||x||_1<r} \min\left((2\pi \sigma)^{-m/2}e^{-||x||_2^2 /2\sigma},(2\pi \sigma)^{-m/2}e^{-||x+f(x)||_2^2 /2\sigma} |I+[\nabla f^{T}](x)|\right) dx\\
&\le 2-2\int_{x:||x||_1<r}(2\pi \sigma)^{-m/2}e^{-(||x||_2^2+B||x||_1^3/2+mB^2||x||_1^4/4)/2\sigma} (1-Bm||x||_1) dx\\
&= 2(2\pi \sigma)^{-m/2}\int_{x:||x||_1<r}e^{-||x||_2^2/2\sigma}-e^{-(||x||_2^2+B||x||_1^3/2+mB^2||x||_1^4/4)/2\sigma} (1-Bm||x||_1) dx\\
&\indent\indent +2(2\pi \sigma)^{-m/2}\int_{x:||x||_1\ge r} e^{-||x||_2^2/2\sigma} dx\\
&= 2(2\pi \sigma)^{-m/2}\int_{x:||x||_1<r}e^{-||x||_2^2/2\sigma}-e^{-(||x||_2^2+B||x||_1^3/2+mB^2||x||_1^4/4)/2\sigma} dx\\
&\indent\indent +2(2\pi \sigma)^{-m/2}\int_{x:||x||_1< r} Bm||x||_1e^{-(||x||_2^2+B||x||_1^3/2+mB^2||x||_1^4/4)/2\sigma} dx\\
&\indent\indent +2(2\pi \sigma)^{-m/2}\int_{x:||x||_1\ge r} e^{-||x||_2^2/2\sigma} dx\\
&\le 2(2\pi \sigma)^{-m/2}\int_{x:||x||_1<r}\frac{2B||x||_1^3+mB^2||x||_1^4}{8\sigma} e^{-||x||_2^2/2\sigma} dx\\
&\indent\indent +2(2\pi \sigma)^{-m/2}\int_{x:||x||_1< r} Bm||x||_1e^{-||x||_2^2/2\sigma} dx+2(2\pi \sigma)^{-m/2}\int_{x:||x||_1\ge r} e^{-||x||_2^2/2\sigma} dx\\
&\le 2(2\pi \sigma)^{-m/2}\int_{x\in\mathbb{R}^m}\frac{2B||x||_1^3+mB^2||x||_1^4}{8\sigma} e^{-||x||_2^2/2\sigma} dx\\
& \indent\indent+2(2\pi \sigma)^{-m/2}\int_{x\in\mathbb{R}^m} Bm||x||_1e^{-||x||_2^2/2\sigma} dx\\
&\indent\indent +2(2\pi \sigma)^{-m/2}(1-Bmr)\int_{x:||x||_1\ge r} e^{-||x||_2^2/2\sigma} dx\\
&\le \frac{m^3B}{2\sigma}\sqrt{8\sigma^3/\pi}+\frac{m^5B^2}{4\sigma}\cdot 3\sigma^2+2m^2B\sqrt{2\sigma/\pi}+2(2\pi \sigma)^{-m/2}(1-Bmr)\int_{x:||x||_1\ge r} e^{-||x||_2^2/2\sigma} dx\\
&= m^3B\sqrt{2\sigma/\pi}+\frac{3m^5B^2\sigma}{4}+2m^2B\sqrt{2\sigma/\pi}+2(2\pi \sigma)^{-m/2}(1-Bmr)\int_{x:||x||_1\ge r} e^{-||x||_2^2/2\sigma} dx\\
&=\frac{4(m+2)m^2B\sqrt{2\sigma/\pi}+3m^5B^2\sigma}{4}+2(2\pi \sigma)^{-m/2}(1-Bmr)\int_{x:||x||_1\ge r} e^{-||x||_2^2/2\sigma} dx
\end{align*}

Next, observe that for any $\lambda\ge 0$, it must be the case that
\begin{align*}
&(2\pi \sigma)^{-m/2}\int_{x:||x||_1\ge r} e^{-||x||_2^2/2\sigma} dx\\
&\le (2\pi \sigma)^{-m/2} e^{-\lambda r/\sigma} \int_{x:||x||_1\ge r} e^{\lambda||x||_1/\sigma} e^{-||x||_2^2/2\sigma} dx\\
&\le (2\pi \sigma)^{-m/2} e^{-\lambda r/\sigma} \int_{x\in\mathbb{R}^m} e^{\lambda||x||_1/\sigma} e^{-||x||_2^2/2\sigma} dx\\
&=e^{-\lambda r/\sigma}\left[ (2\pi \sigma)^{-1/2} \int_{x_1\in\mathbb{R}} e^{\lambda |x_1|/\sigma} e^{-x_1^2/2\sigma}
 dx_1\right]^m\\
&=e^{-\lambda r/\sigma}\left[ 2(2\pi \sigma)^{-1/2} \int_0^\infty e^{\lambda x_1/\sigma} e^{-x_1^2/2\sigma} dx_1\right]^m\\
&=e^{-\lambda r/\sigma}\left[ 2(2\pi \sigma)^{-1/2} \int_0^\infty e^{\lambda ^2/2\sigma} e^{-(x_1-\lambda)^2/2\sigma} dx_1\right]^m\\
&=e^{-\lambda r/\sigma}\left[ 2e^{\lambda ^2/2\sigma} (2\pi \sigma)^{-1/2} \int_{-\lambda}^\infty e^{-x_1^2/2\sigma} dx_1\right]^m\\
&\le e^{-\lambda r/\sigma}\left[ e^{\lambda^2/2\sigma}(1+2\lambda/\sqrt{2\pi\sigma})\right]^m\\
&\le e^{-\lambda r/\sigma+m\lambda^2/2\sigma+2m\lambda/\sqrt{2\pi\sigma}}
\end{align*}

In particular, if we set $\lambda=r/m-\sqrt{2\sigma/\pi}$, this shows that $(2\pi \sigma)^{-m/2}\int_{x:||x||_1\ge r} e^{-||x||_2^2/2\sigma} dx\le e^{-(r/2\sqrt{\sigma}-m/\sqrt{2\pi})^2/m}$. The desired conclusion follows.
\end{proof}



\begin{lemma}
Let $\sigma, B_1, B_2>0$, $m$ be a positive integer with $m<1/B_1$, $m\sqrt{2\sigma/\pi}<r\le(1-mB_1)/(mB_2)$, and $f:\mathbb{R}^m\rightarrow\mathbb{R}^m$ such that $|\frac{\partial f_i}{\partial x_j}(0)|\le B_1$ for all $i$ and $j$, and $|\frac{\partial^2 f_i}{\partial x_j\partial x_{j'}}(x)|\le B_2$ for all $i$, $j$, $j'$, and all $x$ with $||x||_1<r$. Next, let $P$ be the probability distribution of $X+f(X)$ when $X\sim \mathcal{N}(0,\sigma I)$. Then $P$ is a $((1-mB_1)^2\sigma,\epsilon)$-blurring for $\epsilon=\frac{4(m+2)m^2B_2\sqrt{2\sigma/\pi}/(1-mB_1)+3m^5B_2^2\sigma/(1-mB_1)^2}{8}+(1-(1+mB_1)B_2mr)e^{-(r/2\sqrt{\sigma}-m/\sqrt{2\pi})^2/m}$.
\end{lemma}

\begin{proof}
First, define $h:\mathbb{R}^m\rightarrow\mathbb{R}^m$ such that $h(x)=f(0)+x+[\nabla f^{(t)}]^T(0) x$ for all $x$. Every eigenvalue of $[\nabla f](0)$ has a magnitude of at most $mB_1$, so $h$ is invertible. Next, define $f^\star:\mathbb{R}^m\rightarrow\mathbb{R}^m$ such that $f^\star(x)=h^{-1}(x+f(x))-x$ for all $x$. Clearly, $f^\star(0)=0$, and $\frac{\partial f^\star_i}{\partial x_j}(0)=0$ for all $i$ and $j$. Furthermore, for any given $x$ it must be the case that
$\max_{i,j,j'} |\frac{\partial^2 f_i}{\partial x_j\partial x_{j'}}|\ge (1-mB_1) \max_{i,j,j'} |\frac{\partial^2 f^\star_i}{\partial x_j\partial x_{j'}}|$. So, $|\frac{\partial^2 f_i}{\partial x_j\partial x_{j'}}|\le B_2/(1-mB_1)$ for all $i$, $j$, $j'$, and all $x$ with $||x||_1<r$. Now, let $P^\star$ be the probability distribution of $x+f^\star(x)$ when  $x\sim \mathcal{N}(0,\sigma I)$. By the previous lemma, $P^\star$ is a $(\sigma,\epsilon)$-blurring for $\epsilon=\frac{4(m+2)m^2B_2\sqrt{2\sigma/\pi}/(1-mB_1)+3m^5B_2^2\sigma/(1-mB_1)^2}{8}+(1-(1+mB_1)B_2mr)e^{-(r/2\sqrt{\sigma}-m/\sqrt{2\pi})^2/m}$.

Now, let $\widehat{P^\star}$ be a probability distribution such that $P^\star$ is a $(\sigma,\epsilon)$-blurring of $\widehat{P^\star}$. Next, let $\widehat{P}$ be the probability distribution of $h(x)$ when $x$ is drawn from $\widehat{P^\star}$. Also, let $M= (I+[\nabla f^T]^T(0))(I+[\nabla f^T](0))$. The fact that $||P^\star-\widehat{P^\star} * \mathcal{N}(0,\sigma I)||_1\le 2\epsilon$ implies that 
\[||P-\widehat{P} * \mathcal{N}(0,\sigma M)||_1\le 2\epsilon\]
For any $x\in\mathbb{R}^m$, it must be the case that 
\begin{align*}
x\cdot M x&\ge ||x||_2^2-2B_1||x||_1^2-mB_1^2||x||_1^2\\
&\ge ||x||_2^2-2mB_1||x||_2^2-m^2B_1^2||x||_2^2=(1-mB_1)^2||x||_2^2\\
\end{align*}
That in turn means that $\sigma M-\sigma(1-mB_1)^2I$ is positive semidefinite. So, $\widehat{P} * \mathcal{N}(0,\sigma M)=\widehat{P} * \mathcal{N}(0,\sigma M-\sigma(1-mB_1)^2I) * \mathcal{N}(0,\sigma(1-mB_1)^2I)$, which proves that $P$ is a $((1-mB_1)^2\sigma,\epsilon)$-blurring of $\widehat{P} * \mathcal{N}(0,\sigma M-\sigma(1-mB_1)^2I)$. 
\end{proof}

Any blurring is approximately equal to a linear combination of Gaussian distributions, so this should imply a similar result for $X$ drawn from a $(\sigma,\epsilon)$ blurring. However, we are likely to use functions that have derivatives that are large in some places. Not all of the Gaussian distributions that the blurring combines will necessarily have centers that are far enough from the high derivative regions. As such, we need to add an assumption that the centers of the distributions are in regions where the derivatives are small. We formalize the concept of being in a region where the derivatives are small as follows.

\begin{definition}
Let $f:\mathbb{R}^m\rightarrow\mathbb{R}^m$, $x\in\mathbb{R}^m$, and $r, B_1, B_2>0$. Then $f$ is $(r, B_1, B_2)$-stable at $x$ if $|\frac{\partial f_i}{\partial x_j}(0)|\le B_1$ for all $i$ and $j$ and all $x'$ with $||x'-x||_1<r$, and $|\frac{\partial^2 f_i}{\partial x_j\partial x_{j'}}|\le B_2$ for all $i$, $j$, $j'$, and all $x'$ with $||x'-x||_1<2r$. Otherwise, $f$ is $(r, B_1, B_2)$-unstable at $x$.
\end{definition}

This allows us to state the following variant of the previous lemma.

\begin{lemma}\label{stableBlur}
Let $\sigma, B_1, B_2>0$, $m$ be a positive integer with $m<1/B_1$, $m\sqrt{2\sigma/\pi}<r\le(1-mB_1)/(mB_2)$, and $f:\mathbb{R}^m\rightarrow\mathbb{R}^m$ such that there exists $x$ with $||w||_1<r$ such that $f$ is $(r, B_1, B_2)$-stable at $x$. Next, let $P$ be the probability distribution of $X+f(X)$ when $X\sim \mathcal{N}(0,\sigma I)$. Then $P$ is a $((1-mB_1)^2\sigma,\epsilon)$-blurring for $\epsilon=\frac{4(m+2)m^2B_2\sqrt{2\sigma/\pi}/(1-mB_1)+3m^5B_2^2\sigma/(1-mB_1)^2}{8}+(1-(1+mB_1)B_2mr)e^{-(r/2\sqrt{\sigma}-m/\sqrt{2\pi})^2/m}$.
\end{lemma}

\begin{proof}
$|\frac{\partial f_i}{\partial x_j}(0)|\le B_1$ for all $i$ and $j$, and $|\frac{\partial^2 f_i}{\partial x_j\partial x_{j'}}(x')|\le B_2$ for all $i$, $j$, $j'$, and all $x'$ with $||x'||_1<r$. Then, the desired conclusion follows by the previous lemma.
\end{proof}

This lemma could be relatively easily used to prove that if we draw $X$ from a $(\sigma,\epsilon)$-blurring instead of drawing it from $\mathcal{N}(0,\sigma I)$ and $f$ is stable at $X$ with high probability then the probability distribution of $X+f(X)$ will be a $(\sigma',\epsilon')$-blurring for $\sigma'\approx\sigma$ and $\epsilon'\approx\epsilon$. However, that is not quite what we will need. The issue is that we are going to repeatedly apply a transformation along these lines to a variable. If all we know is that its probability distribution is a $(\sigma^{(t)},\epsilon^{(t)})$-blurring in each step, then we potentially have a probability of $\epsilon^{(t)}$ each time step that it behaves badly in that step. That is consistent with there being a probability of $\sum \epsilon^{(t)}$ that it behaves badly eventually, which is too high.

In order to avoid this, we will think of these blurrings as approximations of a $(\sigma,0)$ blurring. Then, we will need to show that if $X$ is good in the sense of being present in the idealized form of the blurring then $X+f(X)$ will also be good. In order to do that, we will need the following definition.

\begin{definition}
Let $P$ be a $(\sigma,\epsilon)$-blurring of $\widehat{P}$, and $X\sim P$. A $\sigma$-revision of $X$ to $\widehat{P}$ is a random pair $(X',M)$ such that the probability distribution of $M$ is $\widehat{P}$, the probability distribution of $X'$ given that $M=\mu$ is $\mathcal{N}(\mu,\sigma I)$, and $P[X'\ne X]=||P-\mathcal{N}(0,\sigma I) *  \widehat{P}||_1/2$. Note that a $\sigma$-revision of $X$ to $\widehat{P}$ will always exist.
\end{definition}

\subsubsection{Means, SLAs, and Gaussian distributions}
Our plan now is to consider a version of NoisyStochasticGradientDescent in which the edge weights get revised after each step and then to show that under suitable assumptions when this algorithm is executed none of the revisions actually change the values of any of the edge weights. Then, we will show that whether the samples are generated randomly or by a parity function has minimal effect on the probability distribution of the edge weights after each step, allowing us to revise the edge weights in both cases to the same probability distribution. That will allow us to prove that the probability distribution of the final edge weights is nearly independent of which probability distribution the samples are drawn from.

The next step towards doing that is to show that if we run NoisySampleGradientDescentStep on a neural network with edge weights drawn from a linear combination of Gaussian distributions, the probability distribution of the resulting graph is essentially independent of what parity function we used to generate the sample. In order to do that, we are going to need some more results on the difficulty of distinguishing an unknown parity function from a random function. First of all, recall that corollary \ref{parityAverage} says that

\begin{corollary}
Let $n>0$ and $f:\B^{n+1}\rightarrow \mathbb{R}$. Also, let $X$ be a random element of $\B^n$ and $Y$ be a random element of $\B$. Then
\[\sum_{s\subseteq[n]} |E[f((X,Y))]-E[f((X,p_s(X)))]|\le 2^{n/2}\sqrt{E[f^2((X,Y))]}\]
\end{corollary}

We can apply this to probability distributions to get the following.

\begin{theorem}
Let $m>0$, and for each $z\in\B^{n+1}$, let $P_z$ be a probability distribution on $\mathbb{R}^m$ with probability density function $f_z$. Now, randomly select $Z\in\B^{n+1}$ and $X\in\B^n$ uniformly and independently. Next, draw $W$ from $P_Z$ and $W'_s$ from $P_{(X,p_s(X))}$ for each $s\subseteq[n]$. Let $P^\star$ be the probability distribution of $W$ and $P^{\star}_s$ be the probability distribution of $W'_s$ for each $s$. Then
\[2^{-n} \sum_{s\subseteq[n]} ||P^\star-P^\star_s||_1\le 2^{-n/2}\int_{\mathbb{R}^m} \max_{z\in \B^{n+1}} f_z(w) dw\]
\end{theorem}

\begin{proof}
Let $f^\star=2^{-n-1}\sum_{z\in\B^{n+1}} f_z$ be the probability density function of $P^\star$, and for each $s\subseteq[n]$, let $f^\star_s=2^{-n}\sum_{x\in\B^n} f_{(x,p_s(x))}$ be the probability density function of $P^\star_s$.

For any $w\in \mathbb{R}^m$, we have that
\begin{align*}
&\sum_{s\subseteq[n]} |f^\star(w)-f^\star_s(w)|\\
&=\sum_{s\subseteq[n]} |E[f_{Z}(w)]-E[f_{(X,p_s(X))}(w)|\\
&\le 2^{n/2}\sqrt{E[f^2_{Z}(w)]}\\
&\le 2^{n/2}\max_{z\in\B^{n+1}} f_z(w)\\
\end{align*}
That means that

\begin{align*}
& \sum_{s\subseteq[n]} ||P^\star-P^\star_s||_1\\
&= \sum_{s\subseteq[n]} \int_{\mathbb{R}^m} |f^\star(w)-f^\star_s(w)| dw\\
&\le \int_{\mathbb{R}^m} \sum_{s\subseteq[n]} |f^\star(w)-f^\star_s(w)| dw\\
&\le 2^{n/2}\int_{\mathbb{R}^m} \max_{z\in \B^{n+1}} f_z(w) dw
\end{align*}
\end{proof}

In particular, if these probability distributions are the result of applying a well-behaved distortion function to a Gaussian distribution, we have the following.

\begin{theorem} \label{gausParity}
Let $\sigma, B_0, B_1>0$, and $n$ and $m$ be positive integers with $m<1/B_1$. Also, for every $z\in\B^{n+1}$, let $f^{(z)}:\mathbb{R}^m\rightarrow\mathbb{R}^m$ be a function such that $|f^{(z)}_i(w)|\le B_0$ for all $i$ and $w$ and $|\frac{\partial f^{(z)}_i}{\partial w_j}(w)|\le B_1$ for all $i$, $j$, and $w$. Now, randomly select $Z\in\B^{n+1}$ and $X\in\B^n$ uniformly and independently. Next, draw $W_0$ from $\mathcal{N}(0,\sigma I)$, set $W=W_0+f^{(Z)}(W_0)$ and $W'_s=W_0+f^{(X,p_s(X))}(W_0)$ for each $s\subseteq[n]$. Let $P^\star$ be the probability distribution of $W$ and $P^\star_s$ be the probability distribution of $W'_s$ for each $s$. Then
\[2^{-n}\sum_{s\subseteq[n]}||P^\star-P^\star_s||_1\le 2^{-n/2}\cdot e^{2m B_0/\sqrt{2\pi\sigma}}/(1-mB_1)\]
\end{theorem}

\begin{proof}
First, note that the bound on $|\frac{\partial f^{(z)}_i}{\partial w_j}(w)|$ ensures that if $w+f^{(z)}(w)=w'+f^{(z)}(w')$ then $w=w'$. So, for any $z$ and $w$, the probability density function of $W_0+f^{(z)}(W_0)$ at $w$ is less than or equal to 
\[(2\pi\sigma)^{-m/2} e^{-\sum_{i=1}^m \max^2(|w_i|-B_0,0)/2\sigma}/ |I+[\nabla f^{(z)}]^T(w)|\]
which is less than or equal to 
\[(2\pi\sigma)^{-m/2} e^{-\sum_{i=1}^m \max^2(|w_i|-B_0,0)/2\sigma}/(1-mB_1)\]

By the previous theorem, that implies that
\begin{align*}
&2^{-n}\sum_{s\subseteq[n]}||P^\star-P^\star_s||_1\\
&\le 2^{-n/2}\int_{\mathbb{R}^m} (2\pi\sigma)^{-m/2} e^{-\sum_{i=1}^m \max^2(|w_i|-B_0,0)/2\sigma}/(1-mB_1) dw\\
&=2^{-n/2}\left[\int_{\mathbb{R}}  (2\pi\sigma)^{-1/2} e^{- \max^2(|w'|-B_0,0)/2\sigma} dw'\right]^m/(1-mB_1)\\
&=2^{-n/2} [1+2B_0/\sqrt{2\pi\sigma}]^m/(1-mB_1)\\
&\le 2^{-n/2}\cdot e^{2m B_0/\sqrt{2\pi\sigma}}/(1-mB_1)
\end{align*}
\end{proof}

The problem with this result is that it requires $f$ to have values and derivatives that are bounded everywhere, and the functions that we will encounter in practice will not necessarily have that property. We can reasonably require that our functions have bounded values and derivatives in the regions we are likely to evaluate them on, but not in the entire space. Our solution to this will be to replace the functions with new functions that have the same value as them in small regions that we are likely to evaluate them on, and that obey the desired bounds. The fact that we can do so is established by the following theorem.

\begin{theorem}
Let $B_0, B_1, B_2, r, \sigma>0$, $\mu\in\mathbb{R}^m$, and $f:\mathbb{R}^m\rightarrow \mathbb{R}^m$ such that there exists $x$ with $||x-\mu||_1< r$ such that $f$ is $(r,B_1, B_2)$-stable at $x$ and $|f_i(x)|\le B_0$ for all $i$. Then there exists a function $f^\star:\mathbb{R}^m\rightarrow\mathbb{R}^m$ such that $f^\star(x)=f(x)$ for all $x$ with $||x-\mu||_1< r$, and $|f_i(x)|\le B_0+2rB_1+2r^2B_2$ and $|\frac{\partial f_i}{\partial x_j}(x)|\le 2B_1+2rB_2$ for all $x\in\mathbb{R}^m$ and $i,j\in[m]$.
\end{theorem}

\begin{proof}
First, observe that the $(r,B_1, B_2)$-stability of $f$ at $x$ implies that for every $x'$ with $||x-x'||\le 2r$, we have that $|\frac{\partial f_i}{\partial x_j}(x')|\le B_1+rB_2$ and $|f_i(x')|\le B_0+2r(B_1+r B_0)$. In particular, this holds for all $x'$ with $||x'-\mu||_1\le 2r-||x-\mu||_1<r$. 

 That means that there exists $r'>r$ such that the values and derivatives of $f$ satisfy the desired bounds for all $x'$ with $||x'-\mu||_1\le r'$. Now, define the function $\overline{f}:\mathbb{R}^m\rightarrow\mathbb{R}^m$ such that $\overline{f}(x')=f(\mu+(x'-\mu)\cdot r'/||x'-\mu||_1)$. This function satisfies the bounds for all $x'$ with $||x'||_1>r'$, except that it may not be differentiable when $x'_j=\mu_j$ for some $j$. Consider defining $f^\star(x')$ to be equal to $f(x')$ when $||x'-\mu||_1\le r'$ and $\overline{f'}(x')$ otherwise. This would almost work, except that it may not be differentiable when $||x'-\mu||_1=r'$, or $||x'-\mu||_1>r'$ and $x'_j=\mu_j$ for some $j$.
 
 In order to fix this, we define a smooth function $h$ of bounded derivative such that $h(x')=0$ whenever $||x'-\mu||_1\le r$, and $h(x')\ge 1$ whenever $||x'-\mu||_1\ge r'$. Then, for all sufficiently small positive constants $\delta$, $f^\star * \mathcal{N}(0,\delta\cdot h^2(x') I)$ has the desired properties.
\end{proof}

Combining this with the previous theorem yields the following.
\begin{corollary}
Let $\sigma, B_0, B_1,B_2,r>0$, $\mu\in\mathbb{R}^m$, and $n$ and $m$ be positive integers with $m<1/(2B_1+2rB_2)$. Then, for every $z\in\B^{n+1}$, let $f^{(z)}:\mathbb{R}^m\rightarrow\mathbb{R}^m$ be a function such that there exists $x$ with $||x-\mu||_1< r$ such that $f$ is $(r,B_1, B_2)$-stable at $x$ and $|f_i(x)|\le B_0$ for all $i$. Next, draw $W_0$ from $\mathcal{N}(\mu,\sigma I)$. Now, randomly select $Z\in\B^{n+1}$ and $X\in\B^n$ uniformly and independently. Then, set $W=W_0+f^{(Z)}(W_0)$ and $W'_s=W_0+f^{(X,p_s(X))}(W_0)$ for each $s\subseteq[n]$. Let $P^\star$ be the probability distribution of $W$ and $P^\star_s$ be the probability distribution of $W'_s$ for each $s$. Then
\[2^{-n}\sum_{s\subseteq[n]}||P^\star-P^\star_s||_1\le 2^{-n/2}\cdot e^{2m ( B_0+2rB_1+2r^2B_2)/\sqrt{2\pi\sigma}}/(1-2mB_1-2rmB_2)+2e^{-(r/2\sqrt{\sigma}-m/\sqrt{2\pi})^2/m}\]
\end{corollary}

\begin{proof}
For each $z$, we can define $f^{(z)\star}$ as an approximation of $f^{(z)}$ as explained in the previous theorem. $||W_0-\mu||_1\le r$ with a probability of at least $1-e^{-(r/2\sqrt{\sigma}-m/\sqrt{2\pi})^2/m}$, in which case $f^{(z)\star}(W_0)=f^{(z)}(W_0)$ for all $z$. For a random $s$, the probability distributions of $W_0+f^{(Z)\star}(W_0)$ and $W_0+f^{(X,p_s(X))\star}(W_0)$ have an $L_1$ difference of at most $2^{-n/2}\cdot e^{2m ( B_0+2rB_1+2r^2B_2)/\sqrt{2\pi\sigma}}/(1-2mB_1-2rmB_2)$ on average by $\ref{gausParity}$. Combining these yields the desired result.
\end{proof}

That finally gives us the components needed to prove the following.

\begin{theorem}
Let $m,n>0$ and define $f^{[z]}:\mathbb{R}^m\rightarrow \mathbb{R}^m$ to be a smooth function for all $z\in \B^{n+1}$. Also, let $\sigma, B_0, B_1, B_2>0$ such that $B_1<1/2m$, $m\sqrt{2\sigma/\pi}<r\le(1-2mB_1)/(2mB_2)$, $T$ be a positive integer, and $\mu_0\in \mathbb{R}^m$. Then, let $\star$ be the uniform distribution on $\B^{n+1}$, and for each $s\subseteq [n]$, let $\rho_s$ be the probability distribution of $(X,p_s(X))$ when $X$ is chosen randomly from $\B^n$. Next, let $P_{\mathcal{Z}}$ be a probability distribution on $\B^{n+1}$ that is chosen by means of the following procedure. First, with probability $1/2$, set $P_{\mathcal{Z}}=\star$. Otherwise, select a random $S\subseteq[n]$ and set $P_{\mathcal{Z}}=\rho_S$.

Now, draw $W^{(0)}$ from $\mathcal{N}(\mu_0,\sigma I)$, independently draw $Z_i\sim P_{\mathcal{Z}}$ and $\Delta^{(i)}\sim \mathcal{N}(0,[2mB_1-m^2B_1^2]\sigma I)$ for all $0<i\le T$. Then, set $W^{(i)}=W^{(i-1)}+f^{[Z_i]}(W^{(i-1)})+\Delta^{(i)}$ for each $0<i\le T$, and let $p$ be the probability that there exists $0\le i\le T$ such that $F^{[Z_i]}$ is $(r, B_1, B_2)$-unstable at $W^{(i)}$ or $||F^{[Z_i]}(W^{(i)})||_\infty>B_0$. Finally, let $Q$ and $Q'_s$ be the probability distribution of $W^{(T)}$ given that $P_{\mathcal{Z}}=\star$ and the probability distribution of $W^{(T)}$ given that $P_{\mathcal{Z}}=\rho_s$. Then
\[2^{-n}\sum_{s\subseteq[n]}||Q-Q'_s||_1\le 4p+T(4\epsilon+\epsilon'+4\epsilon'')\]
where
 \begin{align*}
      &\epsilon=\frac{4(m+2)m^2B_2\sqrt{2\sigma/\pi}/(1-mB_1)+3m^5B_2^2\sigma/(1-mB_1)^2}{8}\\
 &+(1-(1+mB_1)B_2mr)e^{-(r/2\sqrt{\sigma}-m/\sqrt{2\pi})^2/m}
  \end{align*}

\[\epsilon'=2^{-n/2}\cdot e^{2m ( B_0+2rB_1+2r^2B_2)/\sqrt{2\pi\sigma}}/(1-2mB_1-2rmB_2)+2e^{-(r/2\sqrt{\sigma}-m/\sqrt{2\pi})^2/m}\]
\[\epsilon''=e^{-(r/2\sqrt{\sigma}-m/\sqrt{2\pi})^2/m}\]
\end{theorem}

\begin{proof}
In order to prove this, we plan to define new variables $\widetilde{W}^{(i)\prime}$ such that $\widetilde{W}^{(i)\prime}=W^{(i)}$ with high probability for each $i$ and the probability distribution of $\widetilde{W}^{(i)\prime}$ is independent of $P_{\mathcal{Z}}$. More precisely, we define the variables $\widetilde{W}^{(i)}$, $\widetilde{W}^{(i)\prime}$, and $\widetilde{M}^{(i)}$ for each $i$ as follows. First, set $\widetilde{M}^{(0)}=\mu_0$ and $\widetilde{W}^{(0)\prime}=\widetilde{W}^{(0)}=W^{(0)}$. 

Next, for a function $f$ and a point $w$, we say that $f$ is quasistable at $w$ if there exists $w'$ such that $||w'-w||_1\le r$, $f^{[Z_i]}$ is $(r, B_1,B_2)$-stable at $w'$, and $||f^{[Z_i]}(w')||_\infty\le B_0$, and that it is quasiunstable at $w$ otherwise.

for each $0<i\le T$, if $f^{[Z_i]}$ is quasistable at $\widetilde{M}^{(i-1)}$, set 
\[\widetilde{W}^{(i)}=\widetilde{W}^{(i-1)\prime}+f^{[Z_i]}(\widetilde{W}^{(i-1)\prime})+\Delta^{(i)}\]
Otherwise, set 
\[\widetilde{W}^{(i)}=\widetilde{W}^{(i-1)\prime}+\Delta^{(i)}\]

Next, for each $\rho$, let $P^{(i)}_{\rho}$ be the probability distribution of $\widetilde{W}^{(i)}$ given that $P_{\mathcal{Z}}=\rho$. Then, define $\widehat{P}^{(i)}$ as a probability distribution such that $P^{(i)}_{\star}$ is a $(\sigma,\epsilon_0)$-blurring of $\widehat{P}^{(i)}$ with $\epsilon_0$ as small as possible. Finally, for each $\rho$, if $P_{\mathcal{Z}}=\rho$, let $(\widetilde{W}^{(i)\prime},\widetilde{M}^{(i)})$ be a $\sigma$-revision of $\widetilde{W}^{(i)}$ to $\widehat{P}^{(i)}$.


In order to analyse the behavior of these variables, we will need to make a series of observations. First, note that for every $i$, $\rho$, and $\mu$ the probability distribution of $\widetilde{W}^{(i-1)\prime}$ given that $P_Z=\rho$ and $M^{(i-1)}=\mu$
is $\mathcal{N}(\mu,\sigma I)$. Also, either $f^{[Z_i]}$ is quasistable at $\mu$ or $0$ is quasistable at $\mu$. Either way, the probability distribution of $\widetilde{W}^{(i)}$ under these circumstances must be a $(\sigma,\epsilon)$-blurring by Lemma \ref{stableBlur} and Lemma \ref{addBlur}. That in turn means that $P^{(i)}_\rho$ is a $(\sigma,\epsilon)$ blurring for all $i$ and $\rho$, and thus that $P^{(i)}_\star$ must be a $(\sigma,\epsilon)$ blurring of $\widehat{P}^{(i)}$. Furthermore, by the previous corollary, 

\[2^{-n}\sum_{s\subseteq[n]}||P^{(i)}_\star-P^{(i)}_{\rho_s}||_1\le \epsilon'\]

The combination of these implies that 
\[2^{-n}\sum_{s\subseteq[n]}||\mathcal{N}(0,\sigma I) * \widehat{P}^{(i)}-P^{(i)}_{\rho_s}||_1\le 2\epsilon+\epsilon'\]
which in turn means that $P[\widetilde{W}^{(i)\prime}\ne \widetilde{W}^{(i)}]\le \epsilon+\epsilon'/4$. That in turn means that with probability at least $1-T(\epsilon+\epsilon'/4)$ it is the case that $\widetilde{W}^{(i)\prime}= \widetilde{W}^{(i)}$ for all $i$.

If $\widetilde{W}^{(i)\prime}= \widetilde{W}^{(i)}$ for all $i$ and $\widetilde{W}^{(T)\prime}\ne W^{(T)}$ then there must exist some $i$ such that $\widetilde{W}^{(i-1)\prime}= W^{(i-1)}$ but $\widetilde{W}^{(i)}\ne W^{(i)}$. That in turn means that
\begin{align*}
\widetilde{W}^{(i)} &\ne W^{(i)}\\
&=W^{(i-1)}+f^{[Z_i]}(W^{(i-1)})+\Delta^{(i)}\\
&=\widetilde{W}^{(i-1)\prime}+f^{[Z_i]}(\widetilde{W}^{(i-1)\prime})+\Delta^{(i)}
\end{align*}

If $F^{[Z_i]}$ were quasistable at $M^{(i-1)}$, that is exactly the formula that would be used to calculate $\widetilde{W}^{(i)}$, so $F^{[Z_i]}$ must be quasiunstable at $M^{(i-1)}$. That in turn requires that either $F^{[Z_i]}$ is $(r, B_1,B_2)$-unstable at $\widetilde{W}^{(i-1)\prime}= W^{(i-1)}$,  $||f^{[Z_i]}(W^{(i-1)})||_\infty> B_0$, or $||\widetilde{W}^{(i-1)\prime}-M^{(i-1)}||_1>r$. With probability at least $1-p$, neither of the first two scenarios occur for any $i$, while for any given $i$ the later occurs with a probability of at most $\epsilon''$. Thus, 
\[P[\widetilde{W}^{(T)\prime}\ne W^{(T)}]\le p+T(\epsilon+\epsilon'/4+\epsilon'')\]

The probability distribution of $\widetilde{W}^{(T)\prime}$ is independent of $P_{\mathcal{Z}}$, so it must be the case that
\[2^{-n}\sum_{s\subseteq[n]}||Q-Q'||_1\le 2P[\widetilde{W}^{(T)\prime}\ne W^{(T)}|P_{\mathcal{Z}}=\star]+2P[\widetilde{W}^{(T)\prime}\ne W^{(T)}|P_{\mathcal{Z}}\ne \star]\le 4p+T(4\epsilon+\epsilon'+4\epsilon'')\]
\end{proof}

In particular, if we let $(h,G)$ be a neural net, $G_W$ be $G$ with its edge weights changed to the elements of $W$, $L$ be a loss function,  
\[\overline{f}^{(x,y)}(W)=L(eval_{(h,G_W)}(x)-y)\]
, and $f^{(x,y)}=-\gamma\nabla \overline{f}^{(x,y)}$ for each $x,y$ then this translates to the following.

\begin{corollary}
 Let $(h,G)$ be a neural net with $n$ inputs and $m$ edges, $G_W$ be $G$ with its edge weights changed to the elements of $W$, and $L$ be a loss function.
Also, let $\gamma, \sigma, B_0, B_1, B_2>0$ such that $B_1<1/2m$, $m\sqrt{2\sigma/\pi}<r\le(1-2mB_1)/(2mB_2)$, and $T$ be a positive integer. Then, let $\star$ be the uniform distribution on $\B^{n+1}$, and for each $s\subseteq [n]$, let $\rho_s$ be the probability distribution of $(X,p_s(X))$ when $X$ is chosen randomly from $\B^n$. Next, let $P_{\mathcal{Z}}$ be a probability distribution on $\B^{n+1}$ that is chosen by means of the following procedure. First, with probability $1/2$, set $P_{\mathcal{Z}}=\star$. Otherwise, select a random $S\subseteq[n]$ and set $P_{\mathcal{Z}}=\rho_S$.

Now, let $G'$ be $G$ with each of its edge weights perturbed by an independently generated variable drawn from $\mathcal{N}(0,\sigma I)$ and run 
\\$NoisyStochasticGradientDescentAlgorithm(h, G', P_{\mathcal{Z}}, L, \gamma, \infty, \mathcal{N}(0,[2mB_1-m^2B_1^2]\sigma I), T)$. Then, let $p$ be the probability that there exists $0\le i<T$ such that at least one of the following holds: \begin{enumerate}
\item One of the first derivatives of $L(eval_{(h,G_i)}(X_i)-Y_i)$ with respect to the edge weights has magnitude greater than $B_0/\gamma$. 
\item There exists a perturbation $G'_i$ of $G_i$ with no edge weight changed by more than $r$ such that one of the second derivatives of $L(eval_{(h,G'_i)}(X_i)-Y_i)$ with respect to the edge weights has magnitude greater than $B_1/\gamma$. 
\item There exists a perturbation $G'_i$ of $G_i$ with no edge weight changed by more than $2r$ such that one of the third derivatives of $L(eval_{(h,G'_i)}(X_i)-Y_i)$ with respect to the edge weights has magnitude greater than $B_2/\gamma$. 
\end{enumerate}
Finally, let $Q$ be the probability distribution of the final edge weights given that $P_{\mathcal{Z}}=\star$ and $Q'_s$ be the probability distribution of the final edge weights given that $P_{\mathcal{Z}}=\rho_s$. Then
\[2^{-n}\sum_{s\subseteq[n]}||Q-Q'_s||_1\le 4p+T(4\epsilon+\epsilon'+4\epsilon'')\]
where
 \begin{align*}
     &\epsilon=\frac{4(m+2)m^2B_2\sqrt{2\sigma/\pi}/(1-mB_1)+3m^5B_2^2\sigma/(1-mB_1)^2}{8}\\
     &+(1-(1+mB_1)B_2mr)e^{-(r/2\sqrt{\sigma}-m/\sqrt{2\pi})^2/m}
      \end{align*} 
\[\epsilon'=2^{-n/2}\cdot e^{2m ( B_0+2rB_1+2r^2B_2)/\sqrt{2\pi\sigma}}/(1-2mB_1-2rmB_2)+2e^{-(r/2\sqrt{\sigma}-m/\sqrt{2\pi})^2/m}\]

\[\epsilon''=e^{-(r/2\sqrt{\sigma}-m/\sqrt{2\pi})^2/m}\]
\end{corollary}

\begin{corollary}
 Let $(h,G)$ be a neural net with $n$ inputs and $m$ edges, $G_W$ be $G$ with its edge weights changed to the elements of $W$, $L$ be a loss function, and $B>0$. Next, define $\gamma$ such that $0<\gamma\le \pi n/80m^2B$, and let $T$ be a positive integer. Then, let $\star$ be the uniform distribution on $\B^{n+1}$, and for each $s\subseteq [n]$, let $\rho_s$ be the probability distribution of $(X,p_s(X))$ when $X$ is chosen randomly from $\B^n$. Next, let $P_{\mathcal{Z}}$ be a probability distribution on $\B^{n+1}$ that is chosen by means of the following procedure. First, with probability $1/2$, set $P_{\mathcal{Z}}=\star$. Otherwise, select a random $S\subseteq[n]$ and set $P_{\mathcal{Z}}=\rho_S$.

Next, set $\sigma=\left(\frac{40m\gamma B}{n}\right)^2/2\pi$. Now, let $G'$ be $G$ with each of its edge weights perturbed by an independently generated variable drawn from $\mathcal{N}(0,\sigma I)$ and run 
\\$NoisyStochasticGradientDescentAlgorithm(h, G', P_{\mathcal{Z}}, L, \gamma, \infty, \mathcal{N}(0,[2mB\gamma-m^2B^2\gamma^2]\sigma I), T)$. Let $p$ be the probability that there exists $0\le i<T$ such that there exists a perturbation $G'_i$ of $G_i$ with no edge weight changed by more than $160m^2 \gamma B/\pi n$ such that one of the first three derivatives of $L(eval_{(h,G_i)}(X_i)-Y_i)$ with respect to the edge weights has magnitude greater than $B$. Finally, let $Q$ be the probability distribution of the final edge weights given that $P_{\mathcal{Z}}=\star$ and $Q'_s$ be the probability distribution of the final edge weights given that $P_{\mathcal{Z}}=\rho_s$. Then
\[2^{-n}\sum_{s\subseteq[n]}||Q-Q'_s||_1\le 4p+T(720m^4B^2\gamma^2/\pi n+14[e/4]^{n/4})\]
\end{corollary}

\begin{proof}
First, set $r=80m^2\gamma B/\pi n$. Also, set $B_1=B_2=B_3=\gamma B$. By the previous corollary, we have that 
\[2^{-n}\sum_{s\subseteq[n]}||Q-Q'_s||_1\le 4p+T(4\epsilon+\epsilon'+4\epsilon'')\]
where
 \begin{align*}
      &\epsilon=\frac{4(m+2)m^2B_2\sqrt{2\sigma/\pi}/(1-mB_1)+3m^5B_2^2\sigma/(1-mB_1)^2}{8}\\
      &+(1-(1+mB_1)B_2mr)e^{-(r/2\sqrt{\sigma}-m/\sqrt{2\pi})^2/m}
      \end{align*}

\[\epsilon'=2^{-n/2}\cdot e^{2m ( B_0+2rB_1+2r^2B_2)/\sqrt{2\pi\sigma}}/(1-2mB_1-2rmB_2)+2e^{-(r/2\sqrt{\sigma}-m/\sqrt{2\pi})^2/m}\]
\[\epsilon''=e^{-(r/2\sqrt{\sigma}-m/\sqrt{2\pi})^2/m}\]

If $720m^4B^2\gamma^2/\pi n\ge 2$, then the conclusion of this corollary is uninterestingly true. Otherwise, $\epsilon\le 180m^4\gamma^2B^2/\pi n+\epsilon''$. Either way, $\epsilon'\le 4[e/4]^{n/4}+2\epsilon''$, and $\epsilon''\le e^{-m/2\pi}$. $m\ge n$ and $e^{1/2\pi}\ge [4/e]^{1/4}$, so $\epsilon''\le [e/4]^{n/4}$. The desired conclusion follows.
\end{proof}

That allows us to prove the following elaboration of theorem $3$.

\begin{theorem}
Let $(h,G)$ be a neural net with $n$ inputs and $m$ edges, $G_W$ be $G$ with its edge weights changed to the elements of $W$, $L$ be a loss function, and $B>0$. Next, define $\gamma$ such that $0<\gamma\le \pi n/80m^2B$, and let $T$ be a positive integer. Then, for each $s\subseteq [n]$, let $\rho_s$ be the probability distribution of $(X,p_s(X))$ when $X$ is chosen randomly from $\B^n$. Now, select $S\subseteq[n]$ at random. Next, set $\sigma=\left(\frac{40m\gamma B}{n}\right)^2/2\pi$. Now, let $G'$ be $G$ with each of its edge weights perturbed by an independently generated variable drawn from $\mathcal{N}(0,\sigma I)$ and run 
\\$NoisyStochasticGradientDescentAlgorithm(h, G', P_{\mathcal{Z}}, L, \gamma, \infty, \mathcal{N}(0,[2mB\gamma-m^2B^2\gamma^2]\sigma I), T)$. 
Let $p$ be the probability that there exists $0\le i<T$ such that there exists a perturbation $G'_i$ of $G_i$ with no edge weight changed by more than $160m^2 \gamma B/\pi n$ such that one of the first three derivatives of $L(eval_{(h,G_i)}(X_i)-Y_i)$ with respect to the edge weights has magnitude greater than $B$.
For a random $X\in\B^n$, the probability that the resulting net computes $p_S(X)$ correctly is at most $1/2+2p+T(360m^4B^2\gamma^2/\pi n+7[e/4]^{n/4})$.
\end{theorem}

\begin{proof}
Let $Q'_s$ be the probability distribution of the resulting neural net given that $S=s$, and let $Q$ be the probability distribution of the net output by NoisyStochasticGradientDescentAlgorithm $(h, G', \star, L, \gamma, \infty, \mathcal{N}(0,[2mB\gamma-m^2B^2\gamma^2]\sigma I), T)$, where $\star$ is the uniform distribution on $\B^{n+1}$. Also, for each $(x,y)\in\B^{n+1}$, let $R_(x,y)$ be the set of all neural nets that output $y$ when given $x$ as input. The probability that the neural net in question computes $p_S(X)$ correctly is at most
\begin{align*}
&2^{-2n}\sum_{s\subseteq[n],x\in\B^n} P_{G\sim Q'_s}[(f,G)\in R_{x,p_s(y)}]\\
&\le 2^{-2n}\sum_{s\subseteq[n],x\in\B^n} P_{G\sim Q}[(f,G)\in R_{x,p_S(x)}]+||Q-Q'_s||_1/2\\ 
&\le 1/2+2p+T(360m^4B^2\gamma^2/\pi n+7[e/4]^{n/4})
\end{align*}
\end{proof}

\section{Universality of deep learning}\label{universality}
In previous sections, we were attempting to show that under some set of conditions a neural net trained by SGD is unable to learn a function that is reasonably learnable. However, there are some fairly reasonable conditions under which we actually can use a neural net trained by SGD to learn any function that is reasonably learnable. More precisely, we claim that given any probability distribution of functions from $\B^n\to \B$ such that there exists an algorithm that learns a random function drawn from this distribution with accuracy $1/2+\epsilon$ using a polynomial amount of time, memory, and samples, there exists a series of polynomial-sized neural networks that can be constructed in polynomial time and that can learn a random function drawn from this distribution with an accuracy of at least $1/2+\epsilon$ after being trained by SGD on a polynomial number of samples.

\subsection{Emulation of Arbitrary Algorithms 1}
It is easiest to show that we can emulate an arbitrary learning algorithm by running SGD on a neural net if we do not require the neural net to be normal or require a low learning rate. In that case, we use the following argument. Any learning algorithm that uses polynomial time and memory can be reexpressed as a circuit that computes a predicted output and new values for its memory given the actual output from the old values of its memory and an input.

Our neural net for emulating this algorithm will work as follows. First of all, for every bit of memory that the original algorithm uses, $b_i$, our net will have a corresponding vertex $v_{b[i]}$ such that the weight of the edge from the constant vertex to $v_{b[i]}$ will encode the current value of the bit. Secondly, the net will contain a section that computes the predicted value of the output, and the desired new values of each input contingent on each possible value of the actual output. This section will also work in such a way that when the net is evaluated the values of the section's output vertices will always be generated by computing the evaluation function on an input that it has a derivative of $0$ on in order to ensure that nothing can backpropagate through this section, and thus that its edge weights will never change. Thirdly, the net will contain a section that in each step decides whether the net will actually try to get the output right, or whether it will try to learn more about the function. Generally, this section will attempt to learn for a fixed number of steps and then start seriously estimating the output. This section will also be backpropagation proofed.

Fourth, there will be vertices with edges from the vertex computing the desired output for the network and edges to the official output vertex, with edge weights chosen such that if it sets the output correctly the edge weights do not change, and if it sets it incorrectly, these edge weights change in ways that effectively cancel out. Fifth, for every bit of memory that the original algorithm uses, there will be two paths connecting $v_{b[i]}$ to the output vertex. The middle vertices of these paths will also have control paths leading from the part of the network that determines what to change the values in memory to. It will be set up in such a way that these vertices will always evaluate to $0$, but depending on the values of the vertices they are connected to by the control paths, their values might or might not have a nonzero derivative with respect to the input they receive from $v_{b[i]}$. This will allow the network to control whether or not backpropagation can change the value of $b_i$. 

In steps where the network tries to learn more about the function, the network will randomly guess an output, set the value of the output vertex to the opposite of the value it guessed, and ensure that the output has a nonzero derivative with respect to any bits it wants to change the value of if it guessed correctly. That way, if it guesses wrong about the output, the loss function and its derivative are $0$ so nothing changes. If it guesses right, then the values in memory change in exactly the desired manner. No matter what the true output is, it has a $1/2$ chance of guessing right, so it is still updating based on samples drawn from the correct probability distribution. So, running SGD on the neural network described emulates the desired algorithm.



\subsection{Emulation of Arbitrary Algorithms 2}
Of course, the previous result uses choices of a neural net and SGD parameters that are in many ways unreasonable. The activation function is badly behaved, many of the vertices do not have edges from the constant vertex, and the learning rate is deliberately chosen to be so high that it keeps overshooting the minima. If one wanted to do something normal with a neural net trained by SGD one is unlikely to do it that way, and using it to emulate an algorithm is much less efficient than just running the algorithm directly, so this is unlikely to come up.

In order to emulate a learning algorithm with a more reasonable neural net and choice of parameters, we will need to use the following ideas in addition to the ideas from the previous result. First of all, we can control which edges tend to have their weights change significantly by giving edges that we want to change a very low starting weight and then putting high weight edges after them to increase the derivative of the output with respect to them. Secondly, rather than viewing the algorithm we are trying to emulate as a fixed circuit, we will view it as a series of circuits that each compute a new output and new memory values from the previous memory values and the current inputs. Thirdly, a lower learning rate and tighter restrictions on how quickly the network can change prevent us from setting memory values in one step. Instead, we initialize the memory values to a local maximum so that once we perturb them, even slightly, they will continue to move in that direction until they take on the final value. Fourth, in most steps the network will not try to learn anything, so that with high probability all memory values that were set in one step will have enough time to stabilize before the algorithm tries to adjust anything else. Finally, once we have gotten to the point that the algorithm is ready to approximate the function, its estimates will be connected to the output vertex, and the output will gradually become more influenced by it over time as a basic consequence of SGD.


\section{Acknowledgements}
This research was partly supported by the NSF CAREER Award CCF-1552131 and the Google Faculty Research Award.
We thank Jean-Baptiste Cordonnier for helping with the experiment of Section \ref{exp}, as well as Sanjeev Arora, Etienne Bamas, Sebastien Bubeck, Vitaly Feldman, Ran Raz, Oded Regev, Ola Svensson, Santosh Vempala for useful comments and/or references.

\bibliographystyle{amsalpha}
\bibliography{deep}

\end{document}